\documentclass[11pt]{article}

\usepackage{fullpage,times,url,natbib}

\usepackage{amsthm,amsfonts,amsmath,amssymb,epsfig,color,float,graphicx,verbatim}
\usepackage{algorithm,algorithmic}

\usepackage{hyperref}
\hypersetup{
	colorlinks   = true, 
	urlcolor     = blue, 
	linkcolor    = blue, 
	citecolor   = blue 
}

\newtheorem{theorem}{Theorem}

\newtheorem{lemma}{Lemma}
\newtheorem{corollary}{Corollary}

\newtheorem{definition}{Definition}
\newtheorem{remark}{Remark}

\newcommand{\reals}{\mathbb{R}}
\newcommand{\E}{\mathbb{E}}

\newcommand{\be}{\mathbf{e}}
\newcommand{\bx}{\mathbf{x}}
\newcommand{\bw}{\mathbf{w}}

\newcommand{\bu}{\mathbf{u}}
\newcommand{\bv}{\mathbf{v}}
\newcommand{\bz}{\mathbf{z}}

\newcommand{\by}{\mathbf{y}}

\newcommand{\bepsilon}{\boldsymbol{\epsilon}}

\newcommand{\Ocal}{\mathcal{O}}

\newcommand{\Xcal}{\mathcal{X}}

\newcommand{\Fcal}{\mathcal{F}}
\newcommand{\Hcal}{\mathcal{H}}
\newcommand{\Rcal}{\mathcal{R}}
\newcommand{\Ncal}{\mathcal{N}}

\newcommand{\Tcal}{\mathcal{T}}

\newcommand{\norm}[1]{\|#1\|}
\newcommand{\inner}[1]{\langle#1\rangle}

\newtheorem{example}{Example}

\newcommand{\secref}[1]{Sec.~\ref{#1}}
\newcommand{\subsecref}[1]{Subsection~\ref{#1}}

\renewcommand{\eqref}[1]{Eq.~(\ref{#1})}
\newcommand{\lemref}[1]{Lemma~\ref{#1}}

\newcommand{\thmref}[1]{Thm.~\ref{#1}}

\title{The Sample Complexity of One-Hidden-Layer Neural Networks}
\date{}
\author{
	Gal Vardi\thanks{Toyota Technological Institute at Chicago and the Hebrew University of Jerusalem, \texttt{galvardi@ttic.edu}. Work done while the author was at the Weizmann Institute of Science}
	\and
	Ohad Shamir\thanks{Weizmann Institute of Science, Israel, \texttt{ohad.shamir@weizmann.ac.il}}
	\and
	Nathan Srebro\thanks{Toyota Technological Institute at Chicago, \texttt{nati@ttic.edu}}	
}

\begin{document}

\maketitle

\begin{abstract}
We study norm-based uniform convergence bounds for neural networks, aiming at a tight understanding of how these are affected by the architecture and type of norm constraint, for the simple class of scalar-valued one-hidden-layer networks, and inputs bounded in Euclidean norm. We begin by proving that in general, controlling the spectral norm of the hidden layer weight matrix is insufficient to get uniform convergence guarantees (independent of the network width), while a stronger Frobenius norm control is sufficient, extending and improving on previous work. Motivated by the proof constructions, we identify and analyze two important settings where (perhaps surprisingly) a mere spectral norm control turns out to be sufficient: First, when the network's activation functions are sufficiently smooth (with the result extending to deeper networks); and second, for certain types of convolutional networks. In the latter setting, we study how the sample complexity is additionally affected by parameters such as the amount of overlap between patches and the overall number of patches. 
\end{abstract}

\section{Introduction}

Understanding why large neural networks are able to generalize is one of the most important puzzles in the theory of deep learning. Since sufficiently large neural networks can approximate any function, their success must be due to a strong inductive bias in the learned network weights, which is still not fully understood. 

A useful approach to understand such biases is studying what types of constraints on the network weights can lead to uniform convergence bounds, which ensure that empirical risk minimization will not lead to overfitting. Notwithstanding the ongoing debate on whether uniform convergence can fully explain the learning performance of neural networks \citep{nagarajan2019uniform,negrea2020defense,koehler2021uniform}, these bounds provide us with important insights on what norm-based biases can potentially aid in generalization. For example, for linear predictors, it is well-understood that constraints on the Euclidean norm of the weights imply uniform convergence guarantees independent of the number of parameters. This indicates that minimizing the Euclidean norm (without worrying about the number of parameters) is often a useful inductive bias, whether used explicitly or implicitly, or whether uniform convergence formally holds or not for some specific setup. However, neural networks have a more complicated structure than linear predictors, and we still lack a good understanding of what norm-based constraints imply a good inductive bias. 

In this paper, we study this question in the simple case of scalar-valued one-hidden-layer neural networks, which generally compute functions from $\reals^d$ to $\reals$ of the form
\begin{equation}\label{eq:net}
\bx~\mapsto~ \bu^\top \sigma(W\bx)~,
\end{equation}
with weight matrix $W\in \reals^{n \times d}$, weight vector $\bu$, and a fixed (generally non-linear) activation function $\sigma$. We focus on an Euclidean setting, where the inputs $\bx$ and output weight vector $\bv$ are assumed to have bounded Euclidean norm. Our goal is to understand what kind of norm control on the matrix $W$ is required to achieve uniform convergence guarantees, independent of the underlying distribution and the network width $n$ (i.e., the number of neurons). Previous work clearly indicates that a bound on the spectral norm is generally necessary, but (as we discuss below) does not conclusively imply whether it is also sufficient.

Our first contribution (in \subsecref{subsec:dimfree}) is formally establishing that spectral norm control is generally insufficient to get width-independent sample complexity bounds in high dimensions, by directly lower bounding the fat-shattering number of the predictor class. On the flip side, if we assume that the \emph{Frobenius} norm of $W$ is bounded, then we can prove uniform convergence guarantees, independent of the network width or input dimension. The latter result is based on Rademacher complexity, and extends previous results (e.g., \citep{neyshabur2015norm,golowich2018size}, which crucially required homogeneous activations) to general Lipschitz activations. In \subsecref{subsec:fixdim}, we also prove a variant of our lower bound in the case where the input dimension is fixed, pointing at a possibly interesting regime for which good upper bounds are currently lacking. 

The constructions used in our lower bounds crucially require activation functions which are non-smooth around $0$. Moreover, they employ networks where the matrix $W$ can be arbitrary (as long as its norm is bounded). Motivated by this, we identify and analyze two important settings where these lower bounds can be circumvented, and where a mere spectral norm control \emph{is} sufficient to obtain width-independent guarantees:
\begin{itemize}
    \item The first case (studied in \secref{sec:smooth}) is for networks where the activation function $\sigma$ is sufficiently smooth: Specifically, when it is analytic and the coefficients of its Taylor expansion decay sufficiently rapidly. Some examples include polynomial activations, sigmoidal functions such as the error function, and appropriate smoothings of the ReLU function. Perhaps surprisingly, the smoothness of the activation allows us to prove uniform convergence guarantees that depend only on the spectral norm of $W$ and the structure of the activation function, independent of the network width. Moreover, we can extend our results for deeper networks when the activations is a power function (e.g., quadratic activations). 
    \item A second important case (studied in \secref{sec:conv}) is when the network employs weight-sharing on $W$, as in convolutional networks. Specifically, we consider two variants of one-hidden-layer convolutional networks, one with a linear output layer, and another employing max-pooling. In both cases, we present bounds on the sample complexity that depend only on the spectral norm, and study how they depend on the convolutional architecture of the network (such as the number of patches or their amount of overlap). 
\end{itemize}

Our work leaves open quite a few questions and possible avenues for future research, which we discuss in \secref{sec:conclusions}. All proofs of our results appear in Appendix \ref{app:proofs}.

\subsection*{Related Work}


The literature on the sample complexity of neural networks has rapidly expanded in recent years, and cannot be reasonably surveyed here. In what follows, we discuss only works which deal specifically with the issues we focus on in this paper. 

\textbf{Frobenius vs. spectral norm Control, lower bounds.} Fat-shattering lower bounds for neural networks were developed in \citet{anthony1999neural}, but involve size or dimension dependencies rather than norm control. \citet{bartlett2017spectrally} proved a lower bound on the Rademacher complexity of neural networks, implying that a dependence on the spectral norm is generally necessary. \citet{golowich2018size} extended this to show that a dependence on the network width is also necessary, if only the spectral norm is controlled. However, their construction requires a vector-valued (rather than scalar-valued) output. More importantly, the lower bound is on the Rademacher complexity of the predictor class rather than the fat-shattering dimension, and thus (as we further discuss below) does not necessarily imply that the actual sample complexity with some bounded loss function indeed suffers from such a width dependence. \cite{daniely2019generalization} do provide a fat-shattering lower bound, which implies that neural networks on $\reals^d$ with bounded spectral norm and width at most $d$ can shatter $\tilde{\Omega}(d^2)$ points with constant margin, assuming that the inputs have norm at most $\sqrt{d}$. However, this lower bound does not separate between the input norm bound and the width of the hidden layer (which both scale with $d$), and thus does not clarify the contribution of the network width to the bound. Moreover, their proof technique appears to crucially rely on the input's norm scaling with the dimension, rather than being an independent parameter. 

\textbf{Frobenius vs. spectral norm control, upper bounds.} A width-independent uniform convergence guarantee, depending on the Frobenius norm, has been established in \citet{neyshabur2015norm} for constant-depth networks, and in \citet{golowich2018size} for arbitrary-depth networks. However, these bounds are specific to homogeneous activation functions, whereas we tackle general Lipschitz activations (at least for one-hidden layer networks). Bounds based on other norms include \citet{anthony1999neural,bartlett2017spectrally,liang2016cs229t}, but are potentially more restrictive than the Frobenius norm, or do not lead to width-independence. Also, we note that the bound of \citet{bartlett2017spectrally} has the nice property of depending on the distance to some fixed reference matrix, rather than the norm itself. However, we do not pursue this generalization here as it is not the focus of our work.

\textbf{Sample complexity with smooth activations.} The Rademacher complexity for networks with quadratic activations has been studied in \citet{du2018power}, but assuming Frobenius norm constraints, whereas we show that mere spectral norm constraint is sufficient to bound the Rademacher complexity independent of the network width. The strong influence of the activation function on the sample complexity has been observed in the context of  VC-dimension bounds for binary classification (see \citet[Section 7.2]{anthony1999neural}). However, we are not aware of previous results showing how the smoothness of the activation functions provably affects scale-sensitive bounds such as the Rademacher complexity in our setting.

\textbf{Sample complexity of convolutional networks.}
Norm-based uniform convergence bounds for convolutional networks (including more general ones than the one we study) have been provided in \citet{du2018many,long2019generalization}. However, these bounds either depend on the overall number of parameters, or apply only to average-pooling. For convolutional networks with max-pooling,  \citet{ledent2021norm} provide a norm-based analysis which we build on (see \secref{sec:conv} for details).  
%
\citet{cao2019tight} showed an algorithm-dependent sample complexity of learning one-hidden-layer convolutional networks with non-overlapping filters and general activation functions.
Additional works studying the generalization performance of convolutional networks in settings different than ours include \citet{li2018tighter,arora2018stronger,wei2019improved,hsu2020generalization,pmlr-v161-brutzkus21a}.

\section{Preliminaries}\label{sec:preliminaries}

\textbf{Notation.} 
We use bold-face letters to denote vectors, and let $[m]$ be shorthand for $\{1,\ldots,m\}$. Given a matrix $M$, $M_{i,j}$ is the entry in row $i$ and column $j$. Given a function $\sigma(\cdot)$ on $\reals$, we somewhat abuse notation and let $\sigma(\bx)$ (for a vector $\bx$) or $\sigma(M)$ (for a matrix $M$) denote  applying $\sigma$ element-wise. A special case is when $\sigma(\cdot)=[\cdot]_+=\max\{\cdot,0\}$ is the ReLU function. We use standard big-Oh notation, with $\Omega(\cdot),\Theta(\cdot),\Ocal(\cdot)$ hiding constants and $\tilde{\Omega}(\cdot),\tilde{\Theta}(\cdot),\tilde{\Ocal}(\cdot)$ hiding constants and factors polylogarithmic in the problem parameters.

\textbf{Norms.}
$\norm{\cdot}$ denotes the operator norm: For vectors, it is the Euclidean norm, and for matrices, the spectral norm (i.e., $\norm{M}=\sup_{\bx:\norm{\bx}=1}\norm{M\bx}$).  $\norm{\cdot}_{F}$ denotes the  Frobenius norm (i.e., $\norm{M}_F=\sqrt{\sum_{i,j} M_{i,j}^2}$~). It is well-known that for any matrix $M$, $\norm{M}\leq \norm{M}_F$, so the class of matrices whose Frobenius norm is bounded by some $B$ is a subset of the class of matrices whose spectral norm is bounded by the same $B$. Moreover, if $M$ is an $n\times d$ matrix, then $\norm{M}_F\leq \norm{M}\cdot \sqrt{\min\{n,d\}}$.

\textbf{Network Architecture.} Most of our results pertain to scalar-valued one-hidden-layer networks, of the form $\bx\mapsto \bu^\top \sigma(W\bx)$, where $\bx\in \reals^d$, $W\in \reals^{n\times d}$, $\bu$ is a vector and $\sigma(\cdot)$ is some fixed non-linear function. The \emph{width} of the network is $n$, the number of rows of $W$ (or equivalently, the number of neurons in the hidden layer of the network).

\textbf{Fat-Shattering and Rademacher Complexity.}
When studying lower bounds on the sample complexity of a given function class, we use the following version of its \emph{fat-shattering} dimension:
\begin{definition}
	A class of functions $\Fcal$ on an input domain $\Xcal$ \emph{shatters $m$ points $\{\bx_i\}_{i=1}^{m}\subseteq \Xcal$ with margin $\epsilon$}, if there exist a number $s$, such that for all $\by\in \{0,1\}^m$ we can find some $f\in \Fcal$ such that 
	\[
	\forall i\in [m],~~ f(\bx_i)\leq s-\epsilon ~~~\text{if}~~~ y_i=0~~~\text{and}~~~ f(\bx_i)\geq s+\epsilon~~~\text{if}~~~ y_i=1~.
	\]
	The \emph{fat-shattering dimension} of $\Fcal$ (at scale $\epsilon$) is the cardinality $m$ of the largest set of points in $\Xcal$ for which the above holds. 
\end{definition}
It is well-known that the fat-shattering dimension lower bounds the number of samples needed to learn in a distribution-free learning setting, up to accuracy $\epsilon$ (see for example \citet[Part III]{anthony1999neural}). Thus, by proving the existence of a large set of points shattered by the function class, we get lower bounds on the fat-shattering dimension, which translate to lower bounds on the sample complexity. 

As to upper bounds on the sample complexity, our results utilize the \emph{Rademacher complexity} of a function class $\Fcal$, which for our purposes can be defined as
\[
\Rcal_m(\Fcal)~=~\sup_{\{\bx_i\}_{i=1}^{m}\subseteq \Xcal} \E_{\bepsilon}\left[\sup_{f\in \Fcal}~ \frac{1}{m}\sum_{i=1}^{m}\epsilon_i f_i(\bx_i)\right]~,
\]
where $\bepsilon=(\epsilon_1,\ldots,\epsilon_m)$ is a vector of $m$ independent random variables $\epsilon_i$ uniformly distributed on $\{-1,+1\}$. Upper bounds on the Rademacher complexity directly translate to upper bounds on the sample complexity required for learning $\Fcal$: Specifically, the number of inputs $m$ required to make $\Rcal_m(\Fcal)$ smaller than some $\epsilon$ is generally an upper bound on the number of samples required to learn $\Fcal$ up to accuracy $\epsilon$, using any Lipschitz loss (see \citet{bartlett2002rademacher,shalev2014understanding,mohri2018foundations}). We note that Rademacher complexity bounds can also be easily converted to \emph{margin-based} bounds (where the $0-1$ classification risk is upper-bounded by the proportion of margin violations on the training data) by considering a composition of the hypothesis class with an appropriate ramp loss (which upper bounds the 0-1 loss and lower bounds the margin loss, as was done for example in \cite{bartlett2002rademacher,bartlett2017spectrally}).

We note that although the fat-shattering dimension and Rademacher complexity of the predictor class are closely related, they do no always behave the same: For example, the class of norm-bounded linear predictors $\{\bx\mapsto \inner{\bw,\bx}:\bw\in\reals^d,\norm{\bw}\leq B\}$ has Rademacher complexity $\Theta(B/\sqrt{m})$, implying $\Theta((B/\epsilon)^2)$ samples to make it less than $\epsilon$. In contrast, the fat-shattering dimension of the class is smaller, $\Theta(\min\{d,(B/\epsilon)^2\})$ \citep{anthony1999neural,bartlett2002rademacher}. The reason for this discrepancy is that the Rademacher complexity of the predictor class necessarily scales with the range of the function outputs, which is not necessarily relevant if we use bounded losses (that is, if we are actually interested in the function class of linear predictors composed with a bounded loss). Such bounded losses are common, for example, when we are interested in bounding the expected misclassification error  (see for example \citet{bartlett2002rademacher,bartlett2017spectrally}). For this reason, when considering the predictor class itself, we focus on fat-shattering dimension in our lower bounds, and Rademacher complexity in our upper bounds.

\section{Frobenius Norm Control is Necessary for General Networks}\label{sec:frobenius}

We begin by considering one-hidden-layer networks $\bx\mapsto \bu^\top \sigma(W\bx)$, where $\sigma$ is a function on $\reals$ applied element-wise (such as the ReLU activation function). In Subsection \ref{subsec:dimfree}, we consider the dimension-free case (where we are interested in upper and lower bounds that do not depend on the input dimension $d$). In Subsection \ref{subsec:fixdim}, we consider the case where the dimension $d$ is a fixed parameter.

\subsection{Dimension-Free Bounds}\label{subsec:dimfree}

We focus on the following hypothesis class of scalar-valued, one-hidden-layer neural networks of width $n$ on inputs in $\reals^d$, where $\sigma$ is a function on $\reals$ applied element-wise, and where we only bound the operator norms:
\[
\Hcal_{b,B,n,d}^{\sigma}~:=~ \left\{\bx\mapsto \bu^\top \sigma(W\bx)~:~ \bu\in \reals^n~,~W\in \reals^{n\times d}~,~\norm{\bu}\leq b~,~\norm{W}\leq B\right\}~.
\]

The following theorem shows that if the input dimension is large enough, then under a mild condition on the non-smoothness of $\sigma$ around $0$, the fat-shattering dimension of this class necessarily scales with the network width $n$:

\begin{theorem}\label{thm:lowerbound}
	Suppose that the activation function $\sigma$ (as a function on $\reals$) is $1$-Lipschitz on $[-1,+1]$, and satisfies $\sigma(0)=0$ as well as 
	\begin{equation}\label{eq:sigmaasumption}
		\inf_{\delta\in (0,1)} \left|\frac{\sigma(\delta)+\sigma(-\delta)}{\delta}\right|\geq \alpha
	\end{equation}
	for some $\alpha> 0$.
	
	Then there exist universal constants $c,c'>0$ such that the following hold: For any $b,B,b_x,n,\epsilon>0$, there is some $d_0=\text{poly}(b,B,b_x,n,1/\epsilon)$ such that for any input dimension $d\geq d_0$, $\Hcal^{\sigma}_{b,B,n,d}$ can shatter
	\[
	c \alpha^2 \cdot \frac{(bB b_x)^2 n}{\epsilon^2}
	\]
	points from $\{\bx\in \reals^d:\norm{\bx}\leq b_x\}$ with margin $\epsilon$, provided the expression above is larger than $c'(\frac{1}{\alpha^2}+B^2+n)$.
\end{theorem}

To understand the condition in \eqref{eq:sigmaasumption}, suppose that $\sigma$ has a left-hand derivative $\sigma'_-(0)$ and right-hand derivative $\sigma'_+(0)$ at $0$. Recalling that $\sigma(0)=0$, the condition stated in the theorem implies that
\[
\left|\frac{\sigma(\delta)-\sigma(0)}{\delta}-\frac{\sigma(0)-\sigma(-\delta)}{\delta}\right| ~\geq~ \alpha
\]
for all $\delta>0$. In particular, as $\delta\rightarrow 0$, we get $|\sigma'_+(0)-\sigma'_-(0)| > \alpha$. Thus, $\sigma$ is necessarily non-differentiable at $0$. For example, the ReLU activation function satisfies the assumption in the theorem with $\alpha=1$, and the leaky ReLU function $\sigma(z)=\beta z+(1-\beta)[z]_+$ (with parameter $\beta$) satisfies the assumption with $\alpha=1-\beta$.

\begin{remark}\label{remark:bias}
	The assumption $\sigma(0)= 0$ is without much loss of generality: If $\sigma(0)\neq 0$, then let $\hat{\sigma}(z):=\sigma(z)-\sigma(0)$ be a centering of $\sigma$, and note that our predictors can be rewritten in the form $\bx\mapsto \bu^\top \hat{\sigma}(W\bx)+\sigma(0)\cdot\bu^\top \mathbf{1}$. Thus, our hypothesis class is contained in the hypothesis class of predictors of the form $\bx\mapsto \bu^\top \hat{\sigma}(W\bx)+r$ for some  bounded bias parameter $r\in \reals$. This bias term does not change the fat-shattering dimension, and thus is not of much interest. 
\end{remark}

The theorem implies that with only spectral norm control (i.e. where $\norm{\bu}, \norm{W}$ is bounded), it is impossible to get  bounds independent of the width of the network $n$. Initially, the lower bound might appear surprising, since if the activation function $\sigma$ is the identity, $\Hcal^{\sigma}_{b,B,n,d}$ simply contains linear predictors of norm $\leq bB$, for which the sample complexity / fat-shattering dimension is well known to be $\Ocal(bB/\epsilon^2)$ in high input dimensions, completely independent of $n$ (see discussion in the previous section). Intuitively, the extra $n$ term in the lower bound comes from the fact that for random matrices $M$, $\norm{\sigma(M)}$ can typically be much larger than $\norm{M}$, even when $\sigma$ is a Lipschitz function satisfying $\sigma(0)=0$. To give a concrete example, if $M$ is an $n\times n$ matrix with i.i.d. entries uniform on $\{\pm \frac{1}{\sqrt{n}}\}$, then standard concentration results imply that $\E[\norm{M}]$ is upper-bounded by a universal constant independent of $n$, yet the matrix $\sigma(M)$ (where $\sigma$ is entry-wise absolute value) satisfies $\norm{\sigma(M)}=\sqrt{n}$ (since $\sigma(M)$ is just the constant matrix with value $\frac{1}{\sqrt{n}}$ at every entry). The formal proof (in the appendix) relies on constructing a network involving random weights, so that the spectral norm is small yet the network can return sufficiently large values due to the non-linearity.

\begin{remark}
\thmref{thm:lowerbound} has an interesting connection to the recent work of \citet{bubeck2021law}, which implies that in order to fit $m$ points with bounded norm using a width-$n$ one-hidden-layer neural network $\bx\mapsto \bv^\top \sigma(W\bx)$, the Lipschitz constant of the network (and hence $\norm{\bv}\cdot \norm{W}$) must be generally at least $\Omega(\sqrt{m/n})$. The lower bound in \thmref{thm:lowerbound} implies a related statement in the opposite direction: If we allow $\norm{\bv}\cdot \norm{W}$ to be sufficiently larger than $\sqrt{m/n}$, then there exist $m$ points that can be shattered with constant margin. Thus, we seem to get a good characterization of the expressiveness of one-hidden layer neural networks on finite datasets, as a function of their width and the magnitude of the weights. 
\end{remark}

Considering the lower bound, and noting that $B^2 n$ is an upper bound on $\norm{W}_F^2$ which is tight in the worst-case, the bound suggests that a control over the \emph{Frobenius norm} $\norm{W}_F$ would be sufficient to get width-independent bounds. Indeed, such results were previously known when $\sigma$ is the ReLU function, or more generally, a positive-homogeneous function of degree $1$ \citep{neyshabur2015norm,golowich2018size}, with the proofs crucially relying on that property. In what follows, we will prove such a result for general Lipschitz functions (at least for one-hidden layer networks).

Specifically, consider the following hypothesis class, where the previous spectral norm constraint on $W$ is replaced by a Frobenius norm constraint:
\[
\Fcal_{b,B,n,d}^{\sigma}~:=~ \left\{\bx\mapsto \bu^\top \sigma(W\bx)~:~ \bu\in \reals^n~,~W\in \reals^{n\times d}~,~\norm{\bu}\leq b~,~\norm{W}_F\leq B\right\}~.
\]

\begin{theorem}\label{thm:frobupbound}
	Suppose $\sigma(\cdot)$ (as a function on $\reals$) is $L$-Lipschitz and $\sigma(0)=0$. Then for any $b,B,b_x,n,d,\epsilon>0$, the Rademacher complexity of $\Fcal_{b,B,n,d}^{\sigma}$ on $m$ inputs from $\{\bx\in \reals^d:\norm{\bx}\leq b_x\}$ is at most $\epsilon$, if 
	\[
	m~\geq~c\cdot \frac{(bBb_xL)^2(1+\log^3(m))}{\epsilon^2}
	\]
	for some universal constant $c>0$. Thus, it suffices to have $m=\tilde{\Ocal}\left(\left(\frac{bBb_x L}{\epsilon}\right)^2\right)$.
\end{theorem}

The bound is indeed independent of the network width $n$. Also, the result (as an upper bound on the Rademacher complexity) is clearly tight up to log-factors, since in the special case where $\sigma(z)=L\cdot z$ and we fix $\bu=b\cdot \be_1$, then $\Fcal_{b,B,n,d}^{\sigma}$ reduces to the class of linear predictors with Euclidean norm at most $bBL$ (on data of norm at most $b_x$), whose Rademacher complexity matches the bound above up to log-factors.

\begin{remark}[Connection to Implicit Regularization]
It was recently proved that training neural networks employing homogeneous activations on losses such as the logistic loss, without any explicit regularization, gradient methods are implicitly biased towards models which minimize the squared Euclidean norm of their parameters \citep{lyu2019gradient,ji2020directional}. In our setting of one-hidden-layer networks $\bx\mapsto \bu^\top \sigma(W\bx)$, this reduces to $\norm{\bu}^2+\norm{W}_F^2$. For homogeneous activations, multiplying $\bu$ by some scalar $\alpha$ and dividing  $W$ by the same scalar leaves the network unchanged. Based on this observation, and the fact that $\min_{\alpha \in \reals} \norm{\alpha \bu}^2+\norm{\frac{1}{\alpha}W}_F^2 = 2\norm{\bu}\cdot\norm{W}_F$, it follows that minimizing $\norm{\bu}^2+\norm{W}_F^2$ (under any constraints on the network's outputs) is equivalent to minimizing $\norm{\bu}\cdot\norm{W}_F$ (under the same constraints). Thus, gradient methods are biased towards models which minimize our bound from \thmref{thm:frobupbound} in terms of the norms of $\bu,W$.
\end{remark}

\subsection{Dimension-Dependent Lower Bound}\label{subsec:fixdim}

The bounds presented above are dimension-free, in the sense that the upper bound holds for any input dimension $d$, and the lower bound applies once $d$ is sufficiently large. However, for neural networks the case of $d$ being a fixed parameter is also of interest, since we often wish to apply large neural networks on inputs whose dimensionality is reasonably bounded (e.g., the number of pixels in an image). 

For fixed $d$, and for the predictor class itself (without an additional loss composed), it is well-known that there can be a discrepancy between the fat-shattering dimension and the Rademacher complexity, even for linear predictors (see discussion in \secref{sec:preliminaries}). Thus, although \thmref{thm:frobupbound} is tight as a bound on the Rademacher complexity, one may conjecture that the fat-shattering dimension (and true sample complexity for bounded losses) is actually smaller for fixed $d$. 

In what follows, we focus on the case of the Frobenius norm, and provide a dimension-dependent lower bound on the fat-shattering dimension. We first state the result for a ReLU activation with a bias term (\thmref{thm:lowerbound_refined_frobenius}), and then extend it to the standard ReLU activation under a slightly more stringent condition (Corollary \ref{cor:lowerbound_refined_frobenius}). 

\begin{theorem}\label{thm:lowerbound_refined_frobenius}
	For any $b,B,b_x,n,\epsilon$, and any $d$ larger than some universal constant, there exists a choice of $\beta\in [0,\tilde{\Ocal}(\frac{Bb_x}{\sqrt{dn}})]$ such that the following hold: If $\sigma(z)=[z-\beta]_+$, then $\Fcal^{\sigma}_{b,B,n,d}$ can shatter
	\begin{equation}\label{eq:thmlowerboundrefined}
	\tilde{\Omega}\left(\min\left\{nd, \frac{bB b_x}{\epsilon}\sqrt{d}\right\}\right)
	\end{equation}
	points from $\{\bx\in \reals^d:\norm{\bx}\leq b_x\}$ with margin $\epsilon$, assuming the expression above is larger than $cd$ for some universal constant $c>0$, and where $\tilde{\Omega}$ hides factors polylogarithmic in $d,n,b,B,b_x,\frac{1}{\epsilon}$. 
\end{theorem}

\begin{corollary}\label{cor:lowerbound_refined_frobenius}
The lower bound of \thmref{thm:lowerbound_refined_frobenius} also holds for the standard ReLU activation $\sigma(z)=[z]_+$, if $\beta\leq \frac{Bb_x}{\sqrt{n}}$ (which happens if the input dimension $d$ is larger than a factor polylogarithmic in the problem parameters).
\end{corollary}

The lower bound is the minimum of two terms: The first is $nd$, which is the order of the number of parameters in the network. This term is to be expected, since the fat-shattering dimension of $\Fcal$ is at most the pseudodimension of $\Fcal$, which indeed scales with the number of parameters $nd$ (see \cite{anthony1999neural,bartlett2019nearly}). Hence, we cannot expect to be able to shatter many more than $nd$ points. The second term is norm- and dimension-dependent, and dominates the overall lower bound if the network width $n$ is large enough.  Comparing the theorem with the $\tilde{\Ocal}((bBb_x/\epsilon)^2)$ upper bound from \thmref{thm:frobupbound}, it seems to suggest that having a bounded dimension $d$ may improve the sample complexity compared to the dimension-free case, with a smaller dependence on the norm bounds. However, at the moment we do not have upper bounds which match this lower bound, or even establish that bounds better than \thmref{thm:frobupbound} are possible when the dimension $d$ is small. We leave the question of understanding the sample complexity in the fixed-dimension regime as an interesting problem for future research.

\begin{remark}[No contradiction to upper bound in \thmref{thm:frobupbound}, due to implicit bound on $d$]
\thmref{thm:lowerbound_refined_frobenius} requires that \eqref{eq:thmlowerboundrefined} is at least order of $d$ for the lower bound to be valid. This in turn requires that $\frac{bBb_x}{\epsilon}\sqrt{d}\gg d$, or equivalently $d\ll \left(\frac{bBb_x}{\epsilon}\right)^2$. Thus, the theorem only applies when the dimension $d$ is not too large with respect to the other parameters. We note that this is to be expected: If we allow $d\gg \left(\frac{bBb_x}{\epsilon}\right)^2$ (and $n$ sufficiently large), then the lower bound in \eqref{eq:thmlowerboundrefined} will be larger than $\left(\frac{bBb_x}{\epsilon}\right)^2$, and this would violate the $\tilde{\Ocal}\left((bBb_x/\epsilon)^2\right)$ upper bound implied by \thmref{thm:frobupbound}.
\end{remark}

\section{Spectral Norm Control Suffices for Sufficiently Smooth Activations}\label{sec:smooth}

The lower bounds in the previous section crucially rely on the non-smoothness of the activation functions. Thus, one may wonder whether smoothness can lead to better upper bounds. In this section, we show that perhaps surprisingly, this is indeed the case: For sufficiently smooth activations (e.g., polynomials), one can provide width-independent Rademacher complexity bounds, using only the spectral norm. Formally, we return to the class of one-hidden-layer neural networks with spectral norm constraints,
\[
\Hcal^{\sigma}_{b,B,n,d}~=~ \left\{\bx\mapsto \bu^\top \sigma(W\bx)~:~ \bu\in \reals^n~,~W\in \reals^{n\times d}~,~\norm{\bu}\leq b~,~\norm{W}\leq B\right\}~,
\]
and state the following theorem:
\begin{theorem}\label{thm:smooth_upper}
	Fix some $b,B,b_x,n,d,\epsilon >0$. Suppose $\sigma(z)=\sum_{j=1}^{\infty}a_j z^j$ for some $a_1,a_2,\ldots\in\reals$, simultaneously for all $z:|z|\leq Bb_x$. Then the Rademacher complexity of $\Hcal^{\sigma}_{b,B,n,d}$ on $m$ inputs from $\{\bx\in\reals^d:\norm{\bx}\leq b_x\}$ is at most $\epsilon$, if
	\[
	m~\geq~\left(\frac{b\cdot \tilde{\sigma}(B b_x)}{\epsilon}\right)^2~,~~\text{where}~~ \tilde{\sigma}(z):=\sum_{j=1}^{k}|a_j|z^j
	\]
	(assuming the sum converges). 
\end{theorem}

We note that the conditions imply $\sigma(0)=0$, which is assumed for simplicity (see Remark \ref{remark:bias}). We emphasize that this bound depends only on spectral norms of the network and properties of the activation $\sigma$. In particular, it is independent of the network width $n$ as well as the Frobenius norm of $W$. We also note that the bound is clearly tight in some cases: For example, if $\sigma(\cdot)$ is just the identity function, then $\Hcal^{\sigma}_{b,B,n,d}$ reduces to the class of linear predictors of Euclidean norm at most $bB$, whose Rademacher complexity on inputs of norm at most $b_x$ is well-known to equal $\Theta((bBb_x/\epsilon)^2)$. This also demonstrates that the dependence on the spectral norm $B$ is necessary, even with smooth activations.

The proof of the theorem (in the appendix) depends on algebraic manipulations, which involve `unrolling' the Rademacher complexity as a polynomial function of the network inputs, and employing a certain technical trick to simplify the resulting expression, given a bound on the spectral norm of the weight matrix. 

We now turn to provide some specific examples of $\sigma(\cdot)$ and the resulting expression $\tilde{\sigma}(Bb_x)$ in the theorem (see also Figure \ref{fig:my_label}):

\begin{example}\label{example:poly}
	If $\sigma(z)$ is a polynomial of degree $k$, then $\tilde{\sigma}(Bb_x)=\Ocal((Bb_x)^k)$ for large enough $Bb_x$. 
\end{example}
In the example above, the output values of predictors in the class are at most $\Ocal((Bb_x)^k)$, so it is not surprising that the resulting Rademacher complexity scales in the same manner. 

The theorem also extends to non-polynomial activations, as long as they are sufficiently smooth (although the dependence on $Bb_x$ in $\tilde{\sigma}(Bb_x)$ generally becomes exponential). The following is an example for a sigmoidal activation based on the error function:
\begin{example}\label{example:erf}
	If $\sigma(z)=\text{erf}(rz)$ (where $\text{erf}$ is the error function, and $r>0$ is a scaling parameter), then $\tilde{\sigma}(Bb_x)\leq \frac{2rBb_x}{\sqrt\pi}\exp((rBb_x)^2)$.
\end{example}
\begin{proof}
	We have that $\sigma(z)$ equals
	\[
	\text{erf}(rz) ~=~ \frac{2}{\sqrt{\pi}}\int_{t=0}^{rz}\exp(-t^2)dt ~=~ \frac{2}{\sqrt{\pi}}\int_{t=0}^{rz}\sum_{j=0}^{\infty}\frac{(-t^2)^j}{j!}dt
	~=~ 
	\frac{2}{\sqrt{\pi}}\sum_{j=0}^{\infty}\frac{(-1)^j(rz)^{2j+1}}{j!(2j+1)}
	\]
	and therefore
	\[
	\tilde{\sigma}(z)~=~ \frac{2}{\sqrt{\pi}}\sum_{j=0}^{\infty}\frac{(rz)^{2j+1}}{j!(2j+1)}
	~\leq~ \frac{2rz}{\sqrt{\pi}}\sum_{j=0}^{\infty}\frac{\left((rz)^2\right)^j}{j!}
	~=~ \frac{2rz}{\sqrt{\pi}}\exp\left((rz)^2\right)~.
	\]
\end{proof}

\begin{figure}
    \centering
    \includegraphics[scale=0.6]{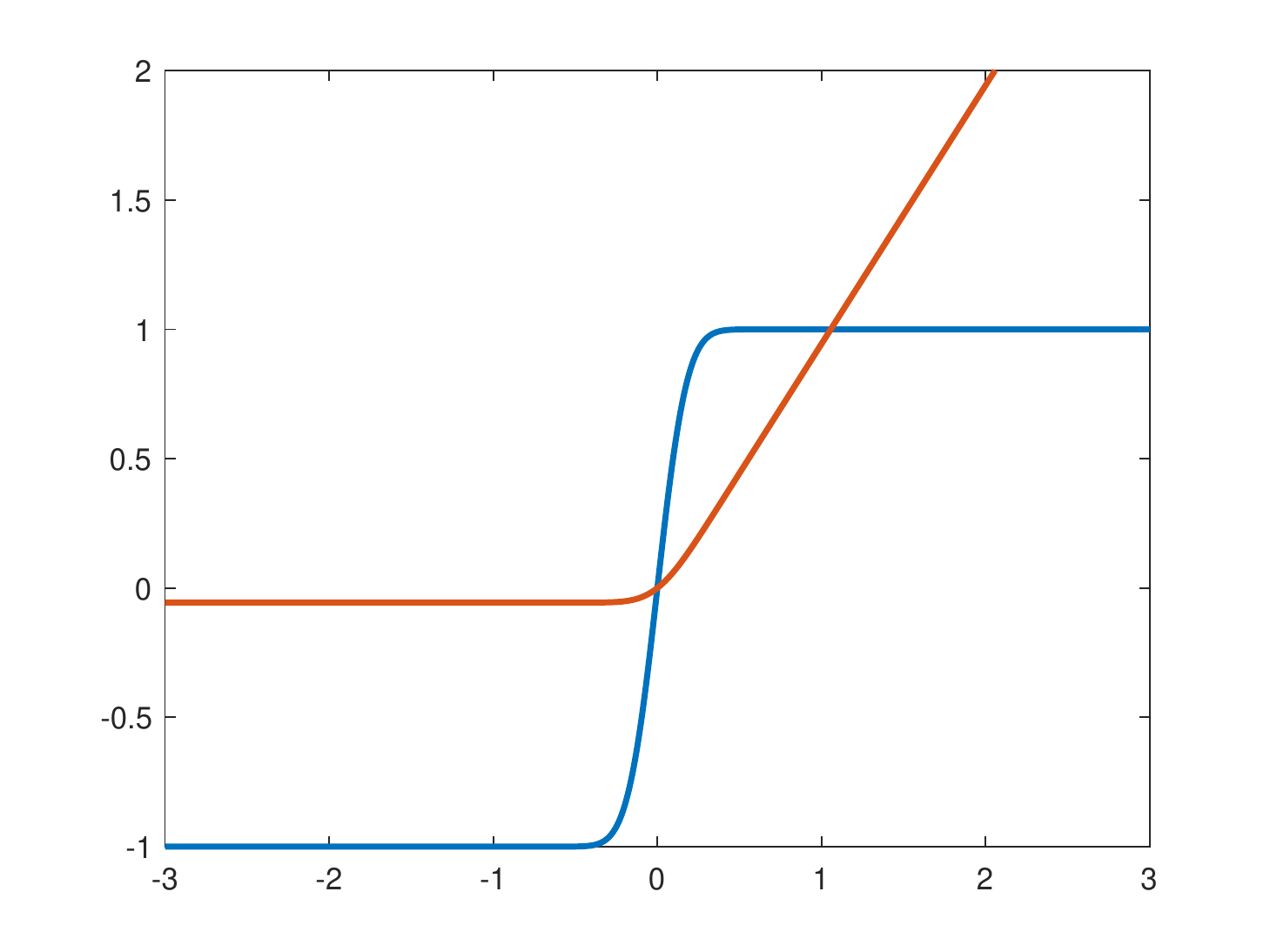}
    \caption{Plots of error function activation from Example \ref{example:erf} (blue) and smoothed ReLU activation from Example \ref{example:relu} (red). Best viewed in color.}
    \label{fig:my_label}
\end{figure}

A sigmoidal activation also allows us to define a smooth approximation of the ReLU function, to which the theorem can be applied: 
\begin{example}\label{example:relu}
	If $\sigma(y) = \frac{1}{2}\left(y+\int_{z=0}^{y}\text{erf}(rz)dz\right)$, then $\tilde{\sigma}(Bb_x)\leq \frac{Bb_x}{2}+\frac{r(Bb_x)^2}{\sqrt{\pi}}\exp((rBb_x)^2)$.
\end{example}
We note that as $r\rightarrow \infty$, $\sigma(y)$ converges uniformly to the ReLU function. 
\begin{proof}
	Using a computation similar to the previous example, $\sigma(y)$ equals
	\[
	\frac{1}{2}y+\frac{1}{\sqrt{\pi}}\sum_{j=0}^{\infty}\frac{(-1)^j(r^{2j+1}y^{2j+2})}{j!(2j+1)(2j+2)},
	\]
	and therefore
	\[
	\tilde{\sigma}(z)~=~ \frac{z}{2}+\frac{1}{\sqrt{\pi}}\sum_{j=0}^{\infty}\frac{r^{2j+1}z^{2j+2}}{j!(2j+1)(2j+2)}~\leq~
	\frac{z}{2}+\frac{rz^2}{\sqrt{\pi}}\sum_{j=0}^{\infty}\frac{\left((rz)^2\right)^j}{j!}~=~
	\frac{z}{2}+\frac{rz^2}{\sqrt{\pi}}\exp((rz)^2)~.
	\]
\end{proof}

Although the last two examples imply an exponential dependence on the spectral norm bound $B$ in the theorem, they still imply that for any fixed $B$, we can get a finite size-independent sample complexity (regardless of the network's width or input dimension) while controlling only the spectral norm of the weight matrices. 

\subsection{Extension to Higher Depths for Power Activations}

When the activation functions are powers of the form $\sigma(z)=z^k$ for some $k$, then the previous theorem can be extended to deeper networks. To formalize this, fix integers $k\geq 1$ and $L\geq 1$, and consider a depth-$(L+1)$ network $f_{L+1}(\bx)$ (parameterized by weight matrices $W^1,W^2,\ldots,W^{L}$ of some arbitrary fixed dimensions, and a weight vector $\bu$) defined recursively as follows:
\[
f_{0}(\bx) = \bx~~,~~ \forall j\in \{0,\ldots,L-1\},~ f_{j+1}(\bx)~=~ (W^{j+1} f_j(\bx))^{\circ k}~~,~~ f_{L+1}(\bx) = \bu^\top f_{L}(\bx)~.
\]
where $(\bv)^{\circ k}$ denotes applying the $k$-th power element-wise on a vector $\bv$. 

\begin{theorem}\label{thm:smooth_upper_deep}
	For any integers $k,L\geq 1$ and choice of matrix dimensions at each layer, consider the class of neural networks $f_{L+1}$ as above, over all weight matrices $W^1\ldots W^{L}$ such that $\norm{W^j}\leq B$ for all $j$, and all $\bu$ such that $\norm{\bu}\leq b$. Then the Rademacher complexity of this class on $m$ inputs from $\{\bx:\norm{\bx}\leq b_x\}$ is at most $\epsilon$, if
	\[
	m~\geq~\left(\frac{b\cdot B^{k+k^2+\ldots k^{L}}\cdot  b_x^{k^{L}}}{\epsilon}\right)^2~.
	\]
\end{theorem}

For constant $k$ and constant-depth networks, the sample complexity bound in the theorem is of the form $b\cdot \text{poly}(B b_x)/\sqrt{m}$, where $B$ bounds merely the (relatively weak) spectral norm. We also note that the exponential/doubly-exponential dependence on $k,L$ is to be expected: Even if we consider networks where each matrix is a scalar $B$, and the input is exactly $b_k$, then multiplying by $B$ and taking the $k$-th power  $L-1$ times over leads to the exact same $B^{k+k^2+\ldots k^{L}}\cdot b_x^{k^{L}}$ factor. Since the Rademacher complexity depends on the scale of the outputs, such a factor is generally unavoidable. The proof of the theorem (in the appendix) builds on the proof ideas of  \thmref{thm:smooth_upper}, which can be extended to deeper networks at least when the activations are  power functions.

\section{Convolutional Networks}\label{sec:conv}

In this section, we study another important example of neural networks which circumvent our lower bounds from \secref{sec:frobenius}, this time by adding additional constraints on the weight matrix. Specifically, we consider one-hidden-layer \emph{convolutional} neural networks. These networks are defined via a set of patches  $\Phi=\{\phi_j\}_{j=1}^{n}$, where for each $j$, the patch $\phi_j:\reals^d\mapsto \reals^{n'}$ projects the input vector $\bx\in \reals^d$ into some subset of its coordinates, namely $\phi_j(\bx)=(x_{i^j_1},\ldots,x_{i^j_{n'}})$ for some $\{i^j_1,\ldots,i^j_{n'}\}\subseteq \{1,\ldots,d\}$. The hidden layer is parameterized by a convolutional filter vector $\bw\in \reals^{n'}$, and given an input $\bx$, outputs the vector $(\sigma(\bw^\top\phi_1(\bx)),\ldots,\sigma(\bw^\top\phi_n(\bx)))\in \reals^n$, where $\sigma$ is some activation function (e.g., ReLU). Note that this can be equivalently written as $\sigma(W\bx)$, where each row  $j$ of $W$ embeds the $\bw$ vector in the coordinates corresponding to $\phi_j(\cdot)$. In what follows, we say that a matrix $W$ \emph{conforms} to a set of patches $\Phi=\{\phi_j\}_{j=1}^{n}$, if there exists a vector $\bw$ such that $(W\bx)_j = \bw^\top \phi_j(\bx)$ for all $\bx$. Thus, our convolutional hidden layer corresponds to a standard hidden layer (same as in previous sections), but with the additional constraint on $W$ that it must conform to a certain set of patches.

In the first subsection below, we study networks where the convolutional hidden layer is combined with a linear output layer. In the following section, we study the case where the hidden layer is combined with a fixed pooling operation. In both cases, we will get bounds that depend on the spectral norm of $W$ and the architecture of the patches. 

\subsection{Convolutional Hidden Layer + Linear Output Layer}

We begin by considering convolutional networks consisting of a convolutional hidden layer (with spectral norm control and with respect to some set of patches), followed by a linear output layer:
\[
\Hcal^{\sigma,\Phi}_{b,B,n,d} = \{\bx\mapsto \bu^\top \sigma(W\bx)~:~\bu\in \reals^n, W\in \reals^{n\times d}, \norm{\bu}\leq b~,~\norm{W}\leq B~,~W~\text{conforms to}~ \Phi\}
\]

The following theorem shows that we can indeed obtain a Rademacher complexity bound depending only on the spectral norm of $W$, and independent of the network width $n$, under a mild assumption about the architecture of the patches:

\begin{theorem}\label{thm:convupbound}
    Suppose $\sigma(\cdot)$ is $L$-Lipschitz and $\sigma(0)=0$. Fix some set of patches $\Phi$, and let $O_{\Phi}$ be the maximal number of patches that any single input coordinate (in $\{1,\ldots,d\}$) appears in. Then for any $b,B,b_x,n,d,\epsilon>0$, the  Rademacher complexity of $\Hcal^{\sigma,\Phi}_{b,B,n,d}$ on $m$ inputs from $\{\bx\in \reals^d:\norm{\bx}\leq b_x\}$ is at most $\epsilon$, if
    \[
    m~\geq~ 2\cdot O_{\Phi}\cdot \left(\frac{bBb_x L}{\epsilon}\right)^2~.
    \]
\end{theorem}

The proof of the theorem (in the appendix) is based on an algebraic analysis of the Rademacher complexity, and the observation that the spectral norm of $W$ necessarily upper bounds the Euclidean norm of the convolutional filter vector $\bw$.

Other than the usual parameters, the bound in the theorem also depends on the architectural parameter $O_{\Phi}\in \{1,\ldots,n\}$, which quantifies the amount of ``overlap'' between the patches. Although it can be as large as $n$ in the worst case (when some single coordinate appears in all patches), for standard convolutional architectures it is usually quite small, and does not scale with the input dimension or the total number of patches. For example, it equals $1$ if the patches are disjoint, and more generally it equals the patch size divided by the stride. Nevertheless, an interesting open question is whether the $O_{\Phi}$ factor in the bound can be reduced or avoided altogether. 


\subsection{Convolutional Hidden Layer + Pooling Layer}

We now turn to consider a slightly different one-hidden-layer convolutional networks, where the linear output layer is replaced by a fixed pooling layer. Specifically, 
we consider networks of the form 
\[
	\bx~\mapsto~ \rho \circ \sigma(W\bx)~=~ \rho \left(\sigma(\bw^\top \phi_1(\bx)),\ldots,\sigma(\bw^\top \phi_n(\bx))\right),
\]
where $\sigma:\reals \to \reals$ is an activation function as before, and $\rho:\reals^n \to \reals$ is $1$-Lipschitz with respect to the \emph{$\ell_\infty$ norm}. For example, $\rho(\cdot)$ may correspond to a max-pooling layer $\bz \mapsto \max_{j \in [n]} z_j$, or to an average-pooling layer $\bz \mapsto \frac{1}{n} \sum_{j \in [n]} z_j$.
We
define the following class of networks:
\begin{align*}
	\Hcal_{B,n,d}^{\sigma, \rho, \Phi}~:=~ 
	\left\{ 
		\bx \mapsto \rho \circ \sigma(W\bx)~:~W \in \reals^{n \times d}~,~\norm{W} \leq B~,~W \text{ conforms to }\Phi
	\right\}~.
\end{align*}

This class is very closely related to a class of convolutional networks recently studied in \citet{ledent2021norm} using an elegant covering number argument. Using their proof technique, we first provide a Rademacher complexity upper bound (\thmref{thm:CNN_upper_bound} below), which depends merely on the spectral norm of $W$, as well as a \emph{logarithmic} dependence on the network width $n$. Although a logarithmic dependence is relatively mild, one may wonder  if we can remove it and get a fully width-independent bound, same as our previous results. Our main novel contribution in this section (\thmref{thm:lower bound cnn}) is to show that this is \emph{not} the case: The fat-shattering dimension of the class necessarily has a $\log(n)$ factor, so the upper bound is tight up to factors polylogarithmic in the sample size $m$.

\begin{theorem}\label{thm:CNN_upper_bound}
	Suppose that $\sigma:\reals \to \reals$ is $L$-Lipschitz and $\sigma(0)=0$, and that $\rho:\reals^n \to \reals$ is $1$-Lipschitz w.r.t. $\ell_\infty$ and satisfies $\rho(\mathbf{0})=0$. Fix some set of patches $\Phi=\{\phi_j\}_{j=1}^n$.
	Then, for any $B,n,d,b_x,\epsilon>0$, the Rademacher complexity of $\Hcal_{B,n,d}^{\sigma, \rho, \Phi}$ on $m$ inputs from $\left\{\bx \in \reals^d~:~\norm{\phi_j(\bx)} \leq b_x \text{ for all } j \in [n]\right\}$ is at most $\epsilon$ if
	\[
		m 
		\geq c \cdot \left( \frac{L B b_x}{\epsilon} \right)^2 \cdot \log^2(m)\log(mn)
	\]
	for some universal constant $c>0$. Thus, it suffices to have $m=\tilde{\Ocal}\left( \left( \frac{L B b_x}{\epsilon} \right)^2 \right)$.
\end{theorem}

For the lower bound, we focus for simplicity on the case where $\rho(\bz)=\max_j z_j$ is a max-pooling layer, and where $\sigma$ is the ReLU function (which satisfies the conditions of \thmref{thm:CNN_upper_bound} with $L=1$). However, we emphasize that unlike the lower bound we proved in \secref{sec:frobenius}, the construction does not rely on the non-smoothness of $\sigma$, and in fact can easily be verified to apply (up to constants) for any $\sigma$ satisfying $\sigma(0)=0$ and $\sigma(\epsilon) \geq c\cdot \epsilon$ (where $c>0$ is a constant).

\begin{theorem} \label{thm:lower bound cnn}
For any $B,n,b_x,\epsilon>0$, there is $d,\Phi$ such that the following hold: The class $\Hcal_{B,n,d}^{\sigma,\rho,\Phi}$, with $\sigma$ being the ReLU function and $\rho$ being the max function, can shatter
\[
	 \frac{1}{4} \cdot 
	  \left( \frac{B b_x}{\epsilon} \right)^2 \cdot \log(n)
\]
points from $\{ \bx \in \reals^d: \norm{\bx} \leq b_x \}$ with margin $\epsilon$.

Moreover, this claim holds already where $\Phi$ corresponds to a convolutional layer with a constant stride $1$, in the following sense:
If we view the input $\bx \in \reals^d$ as a vectorization of a tensor of order $p=\Ocal(\log(n))$, then $\Phi$ corresponds to all patches of certain fixed dimensions $s_1 \times \ldots \times s_p$ in the tensor. 
\end{theorem}

The proof of the theorem is rather technical, but can be informally described as follows (where we focus just on where the dependence on $p=\Ocal(\log(n))$ comes from): We consider each input $\bx$ as a tensor of size $3\times 3\times\ldots \times 3$ (with entries indexed by a vector in $\{1,2,3\}^p)$, and the patches are all sub-tensors of dimensions $2\times 2\times \ldots \times 2$. We construct inputs $\bx_1,\ldots,\bx_p$, where each $\bx_i$ contains zeros, and a single $1$ value at  coordinate $(2,2,\ldots,2,3,2,\ldots,2)$ (with a $3$ at position $i$, and $2$ elsewhere). Given a vector $(y_1,\ldots,y_p)\in \{0,1\}^p$ of target values, we construct the convolutional filter $\bw$ (a $p$-th order tensor of dimensions $2\times 2 \times \ldots \times 2$) to have zeros, and a single $1$ value at coordinate $(y_1+1,\ldots,y_p+1)$. Thus, we ``encode'' the full set of target values in $\bw$, and a simple calculation reveals that given $\bx_i$, the network will output $1$ if $y_i=1$, and $0$ otherwise. Thus, we can shatter $p=\Ocal(\log(n))$ points. An extension of this idea allows us to also incorporate the right dependence on the other problem parameters. 

\section{Conclusions and Open Questions}\label{sec:conclusions}

In this paper, we studied sample complexity upper and lower bounds for one-hidden layer neural networks, based on bounding the norms of the weight matrices. We showed that in general, bounding the spectral norm cannot lead to size-independent guarantees, whereas bounding the Frobenius norm does. However, the constructions also pointed out where the lower bounds can be circumvented, and where a spectral norm control suffices for width-independent guarantees: First, when the activations are sufficiently smooth, and second, for certain types of convolutional networks. 

Our work raises many open questions for future research. For example, how does having a fixed input dimension $d$ affect the sample complexity of neural networks? Our lower bound in \thmref{thm:lowerbound_refined_frobenius} indicates small $d$ might reduce the sample complexity, but currently we do not have good upper bounds that actually establish that (at least without depending on the network width as well). Alternatively, it could also be that \thmref{thm:lowerbound_refined_frobenius} can be strengthened. In a related vein, our lower bound for convolutional networks (\thmref{thm:lower bound cnn}) requires a relatively high dimension, at least on the order of the network width. Can we get smaller bounds if the dimension is constrained?

In a different direction, we showed that spectral norm control does not lead to width-free guarantees with non-smooth activations, whereas such guarantees are possible with very smooth activations. What about other activations? Can we characterize when can we get such guarantees for a given activation function? Or at least, can we improve the dependence on the norm bound for sufficiently smooth non-polynomial activations?

As to convolutional networks, we studied two particular architectures employing weight-sharing: One with a linear output layer, and one with a fixed Lipschitz pooling layer mapping to $\reals$. Even for one-hidden-layer networks, this leaves open the question of characterizing the width-independent sample complexity of networks $\bx\mapsto \bu^\top \rho\circ \sigma(W\bx)$, where $W$ implements weight-sharing and $\rho$ is a pooling operator mapping to $\reals^{p}$ with $p>1$ (\citet{ledent2021norm} provide upper bounds in this setting, but they are not size-independent and we conjecture that they can be improved). Moreover, we still do not know whether parameters such as the amount of overlap in the patches (see \thmref{thm:convupbound}) are indeed necessary. 

All our bounds are in terms of the parameter matrix norm, $\norm{W}$ or $\norm{W}_F$. Some existing bounds, such as in \citet{bartlett2017spectrally}, depend instead on the distance from some fixed data-independent matrix $W_0$ (e.g., the initialization point), a quantity which can be much smaller. We chose to ignore this issue in our paper for simplicity, but it would be useful to generalize our bounds to incorporate this.

Beyond these, perhaps the most tantalizing open question is whether our results can be extended to deeper networks, and what types of bounds we might expect. Even if we treat the depth as a constant, existing bounds for deeper networks are either not width-independent (e.g., \citet{neyshabur2018pac,daniely2019generalization}), utilize norms much stronger than even the Frobenius norm (e.g., \citet{anthony1999neural,bartlett2017spectrally}), or involve products of Frobenius norms, which is quite restrictive \citep{neyshabur2015norm,golowich2018size}. Based on our results, we know that a bound depending purely on the spectral norms is impossible in general, but conjecture that the existing upper bounds are all relatively loose. A similar question can be asked for more specific architectures such as convolutional networks.

\subsubsection*{Acknowledgements}

This research is supported in part by European Research Council (ERC) grant 754705, and NSF-BSF award 1718970. We thank Roey Magen for spotting some bugs in the proof of  \thmref{thm:frobupbound} in a previous version of this paper.

\bibliographystyle{plainnat}
\bibliography{bib}

\appendix

\section{Proofs}\label{app:proofs}
	
\subsection{Proof of \thmref{thm:lowerbound}}

We will assume without loss of generality that the condition $\inf_{\delta\in (0,1)} \left|\frac{\sigma(\delta)+\sigma(-\delta)}{\delta}\right|\geq \alpha$ stated in the theorem holds without an absolute value, namely
\begin{equation}\label{eq:sigmareq}
	\inf_{\delta\in (0,1)} \frac{\sigma(\delta)+\sigma(-\delta)}{\delta}~\geq ~\alpha~.
\end{equation}
To see why, note that if $\inf_{\delta\in (0,1)} \left|\frac{\sigma(\delta)+\sigma(-\delta)}{\delta}\right|\geq \alpha\geq 0$, then $\frac{\sigma(\delta)+\sigma(-\delta)}{\delta}$ can never change sign as a function of $\delta$ (otherwise it will be $0$ for some $\delta$). Hence, the condition implies that either $\frac{\sigma(\delta)+\sigma(-\delta)}{\delta}~\geq ~\alpha$ for all $\delta\in (0,1)$, or that $-\frac{\sigma(\delta)+\sigma(-\delta)}{\delta}~\geq ~\alpha$ for all $\delta\in (0,1)$. We simply choose to treat the first case, as the second case can be treated with a completely identical analysis, only flipping some of the signs.

Fix some sufficiently large dimension $d$ and integer $m\leq d$ to be chosen later. Choose $\bx_1,\ldots,\bx_m$ to be some $m$ orthogonal  vectors of norm $b_x$ in $\reals^d$.
Let $X$ be the $d\times m$ matrix whose $i$-th column is $\bx_i$. 
Given this input set, it is enough to show that there is some number $s$, such that for any $\by\in \{0,1\}^m$, we can find a predictor (namely, $\bu,W$ depending on $\by$) in our class, such that $\norm{\bu}\leq b$, $\norm{W}\leq B$, and 
\begin{equation}\label{eq:shattering}
	\forall i~, \bu^\top \sigma(W\bx_i)~\text{is}~\begin{cases} \leq s-\epsilon& y_i=0\\
		\geq s+\epsilon& y_i=1\end{cases}~.
\end{equation}
We will do so as follows: We let
\[
\bu=\frac{b}{\sqrt{n}}\mathbf{1}~~~\text{and}~~~
W = \frac{\delta}{b_x^2} V\text{diag}(\by)X^\top~,
\]
Where $\delta\in (0,1)$ is a certain scaling factor and $V$ is a $\pm 1$-valued matrix of size $n\times m$, both to be chosen later. In particular, we will assume that $V$ is approximately balanced, in the sense that for any column $i\in [n]$ of $V$, if $p_i$ is the portion of $+1$ entries in the column, then
\begin{equation}\label{eq:Vreq1}
	\max_i \left|\frac{1}{2}-p_i\right|~\leq~\frac{\alpha}{8}~.
\end{equation}

For any $i\in [m]$, since $\bx_1,\ldots,\bx_m$ are orthogonal and of norm $b_x$, we have
\begin{align*}
	\bu^\top \sigma(W\bx_i)~&=~ \bu^\top \sigma\left(\frac{\delta}{b_x^2} V\text{diag}(\by) X^\top \bx_i\right)~=~\bu^\top \sigma(\delta y_i \bv_i)~=~ \frac{b}{\sqrt{n}}\sum_{j=1}^{n} \sigma(\delta y_i V_{j,i})
\end{align*}
where $\bv_i$ is the $i$-th column of $V$, and $V_{j,i}$ is the entry of $V$ in the $j$-th row and $i$-th column. Then we have the following:
\begin{itemize}
	\item If $y_i=0$, this equals $b\sqrt{n}\sigma(0)=0$.
	\item If $y_i=1$, this equals $b\sqrt{n}\left(p_{i}\sigma(\delta)+(1-p_{i})\sigma(-\delta)\right)$, where $p_{i}\in [\frac{1}{2}-\frac{\alpha}{8},\frac{1}{2}+\frac{\alpha}{8}]$ is the portion of entries in the $i$-th column of $V$ with value $+1$. Rewriting it and using \eqref{eq:sigmareq}, \eqref{eq:Vreq1} and the fact that $\sigma(\cdot)$ is $1$-Lipschitz on $[-1,+1]$, we get the expression
	\[
	b\sqrt{n}\left(\frac{\sigma(\delta)+\sigma(-\delta)}{2}-\left(\frac{1}{2}-p_i\right)\left(\sigma(\delta)-\sigma(-\delta)\right)\right)~\geq~  b\sqrt{n}\left(\frac{\delta\alpha}{2}-\frac{\alpha}{8}\cdot 2\delta\right)~=~ \frac{b\sqrt{n}\delta\alpha}{4}~.
	\]
\end{itemize}
Recalling \eqref{eq:shattering}, we get that by fixing $s=\frac{\sqrt{n}\delta \alpha}{8}$, we can shatter the dataset as long as 
\begin{equation}\label{eq:deltareq1}
	\frac{b\sqrt{n}\delta\alpha}{8}\geq \epsilon~~~~\Rightarrow~~~ \delta\geq \frac{8\epsilon}{\alpha b\sqrt{n}}~.
\end{equation}

Leaving this condition for a moment, we now turn to specify how $\delta,V$ is chosen, so as to satisfy the condition $\norm{W}=\norm{\frac{\delta}{b_x^2} V\text{diag}(\by)X^\top}\leq B$. To that end, we let $V$ be any $\pm 1$-valued $n\times m$ matrix which satisfies \eqref{eq:Vreq1} as well as $\norm{V}\leq c(\sqrt{n}+\sqrt{m})$, where $c\geq 1$ is some universal constant. Such a matrix necessarily exists assuming $m$ is sufficiently larger than $\frac{1}{\alpha^2}$\footnote{This follows from the probabilistic method: If we pick the entries of $V$ uniformly at random, then both conditions will hold with some arbitrarily large probability (assuming $m$ is sufficiently larger than $1/\alpha^2$, see for example \citet{seginer2000expected}), hence the required matrix will result with some positive probability.}. It then follows that $\norm{W}\leq \frac{\delta}{b_x^2} \norm{V}\cdot \norm{\text{diag}(\by)}\cdot \norm{X}\leq \frac{\delta}{b_x^2}\cdot c(\sqrt{n}+\sqrt{m})\cdot b_x = \frac{\delta}{b_x}\cdot c(\sqrt{n}+\sqrt{m})$. Therefore, by assuming
\[
\delta\leq \frac{Bb_x}{c(\sqrt{n}+\sqrt{m})},
\]
we ensure that $\norm{W}\leq B$.

Collecting the conditions on $\delta$ (namely, that it is in $(0,1)$, satisfies \eqref{eq:deltareq1}, as well as the displayed equation above), we get that there is an appropriate choice of $\delta$ and we can shatter our $m$ points, as long as $m$ is sufficiently larger than $1/\alpha^2$ and that
\[
1~>~ \frac{B b_x}{c(\sqrt{n}+\sqrt{m})}~\geq~ \frac{8\epsilon}{\alpha b\sqrt{n}}~.
\]
The first inequality is satisfied if (say) we can choose $m\geq (Bb_x/c)^2$ (which we will indeed do in the sequel). As to the second inequality, it is certainly satisfied if $m\geq n$, as well as
\[
\frac{Bb_x }{2c\sqrt{m}}~\geq~ \frac{8\epsilon}{\alpha b\sqrt{n}}~~~\Longrightarrow~~~
m\leq \left(\frac{\alpha}{16c}\right)^2\cdot \frac{(bB b_x)^2 n}{\epsilon^2}~.
\]
Thus, we can shatter any number $m$ of points up to this upper bound. Picking $m$ on this order (assuming it is sufficiently larger than $1/\alpha^2$, $B^2$ or $n$), assuming that the dimension $d$ is larger than $m$, and renaming the universal constants, the result follows.

\subsection{Proof of \thmref{thm:frobupbound}}

To simplify notation, we rewrite $\sup_{\bu,W:\norm{\bu}\leq b,\norm{W}_F\leq B}$ as simply $\sup_{\bu,W}$. Also, we let $\bw_j$ denote the $j$-th row of the matrix $W$.

Fix some set of inputs $\bx_1,\ldots,\bx_m$ with norm at most $b_x$. The Rademacher complexity equals
\begin{align*}
	\E_{\bepsilon} \sup_{\bu,W}&\frac{1}{m}\sum_{i=1}^{m}\epsilon_i \bu^\top\sigma(W\bx_i)~=~
	\E_{\bepsilon} \sup_{\bu,W}\frac{1}{m}\bu^\top\left(\sum_{i=1}^{m}\epsilon_i \sigma(W\bx_i)\right)\\
	&=~\frac{b}{m}\cdot\E_{\bepsilon} \sup_{W}\left\|\sum_{i=1}^{m}\epsilon_i \sigma(W\bx_i)\right\|
	~=~\frac{b}{m}\cdot\E_{\bepsilon} \sup_{W}\sqrt{\sum_{j=1}^{n}\left(\sum_{i=1}^{m}\epsilon_i \sigma( \bw_j^\top\bx_i)\right)^2}~.
\end{align*}
Each matrix in the set $\{W\in \reals^{d\times n}:\norm{W}_F\leq B\}$ is composed of rows, whose sum of squared norms is at most $B^2$. Thus, the set can be equivalently defined as the set of $d\times n$ matrices, where each row $j$ equals $v_j \bw_j$ for some $v_j>0$, $\norm{\bw}_j\leq 1$, and $\norm{(v_1,\ldots,v_n)}^2=\norm{\bv}^2\leq B^2$. Noting that each $v_j$ is positive, we can upper bound the expression in the displayed equation above as follows:
\begin{align}
	&\frac{b}{m}\cdot\E_{\bepsilon} \sup_{\bv,\{\bw_j\}}\sqrt{\sum_{j=1}^{n}\left(\sum_{i=1}^{m}\epsilon_i \sigma(v_j\bw_j^\top\bx_i)\right)^2}\notag\\
	&=~
	\frac{b}{m}\cdot\E_{\bepsilon} \sup_{\bv,\{\bw_j\}}\sqrt{\sum_{j=1}^{n}v_j^2\left(\sum_{i=1}^{m} \frac{\epsilon_i}{v_j}\sigma(v_j\bw_j^\top\bx_i)\right)^2}\notag\\
	&\leq~
	\frac{b}{m}\cdot\E_{\bepsilon} \sup_{\bv,\bv',\{\bw_j\}}\sqrt{\sum_{j=1}^{n}{v'}_j^2\left(\sum_{i=1}^{m} \frac{\epsilon_i}{v_j}\sigma(v_j\bw_j^\top\bx_i)\right)^2}~,\label{eq:frobupbound1}
\end{align}
where $\bv'=(v'_1,\ldots,v'_n)$ satisfies $\norm{\bv'}^2=\sum_{j=1}^{n}{v'}_j^2\leq B^2$ (note that $\bv$ must also satisfy this constraint). Considering this constraint in \eqref{eq:frobupbound1}, we see that for any choice of $\bepsilon,\bv$ and $\bw_1,\ldots,\bw_n$, the supremum over $\bv'$ is clearly attained by letting ${v'}_{j^*} = B$ for some $j^*$ for which $\left(\sum_{i=1}^{m} \frac{\epsilon_i}{v_j}\sigma(v_j\bw_j^\top\bx_i)\right)^2$ is maximized, and ${v'}_{j}=0$ for all $j\neq j*$. Plugging this observation back into \eqref{eq:frobupbound1} and writing the supremum constraints explicitly, we can upper bound the displayed equation by
\begin{align}
	&\frac{bB}{m}\cdot\E_{\bepsilon} \sup_{\bv:\min_j v_j>0,\norm{\bv}\leq B}~~\sup_{\bw_1,\ldots\bw_n:\max_j\norm{\bw_j}\leq 1}\max_j\left|\sum_{i=1}^{m} \frac{\epsilon_i}{v_j}\sigma(v_j\bw_j^\top\bx_i)\right|\notag\\
	&~~~~=~
	\frac{bB}{m}\cdot\E_{\bepsilon} \sup_{v\in (0,B],\bw:\norm{\bw}\leq 1}\left|\sum_{i=1}^{m} \frac{\epsilon_i}{v}\sigma(v\bw^\top\bx_i)\right|\notag\\
	&~~~=~
	\frac{bB}{m}\cdot\E_{\bepsilon} \sup_{v\in (0,B],\bw:\norm{\bw}\leq 1}\left|\sum_{i=1}^{m}\epsilon_i \psi_v\left(\bw^\top\bx_i\right)\right|~,
	\label{eq:frobupbound2}
\end{align}
where $\psi_{v}(z):=\frac{\sigma(vz)}{v}$ for any $z\in \reals$. Since $\sigma(\cdot)$ is $L$-Lipschitz, it follows that $\psi_{\bv}(\cdot)$ is also $L$-Lipschitz regardless of $v$, since for any $z,z'\in \reals$,
\[
|\psi_v(z)-\psi_v(z')|~=~\frac{|\sigma(vz)-\sigma(vz')|}{v}~\leq~ \frac{L|vz-vz'|}{v}~=~ L|z-z'|~.
\]
Thus, the supremum over $v$ in \eqref{eq:frobupbound2} corresponds to a supremum over a class of $L$-Lipschitz functions which all equal $0$ at the origin (since $\psi_v(0)=\frac{\sigma(0)}{v}=0$ by assumption). As a result, we can upper bound \eqref{eq:frobupbound2} by
\[
\frac{bB}{m}\cdot\E_{\bepsilon} \sup_{\psi\in \Psi_L,\bw:\norm{\bw}\leq 1}\left|\sum_{i=1}^{m}\epsilon_i \psi\left(\bw^\top\bx_i\right)\right|~,
\]
where $\Psi_L$ is the class of \emph{all} $L$-Lipschitz functions which equal $0$ at the origin. 

To continue, it will be convenient to get rid of the absolute value in the displayed expression above. This can be done by noting that the expression equals
\begin{align}
&\frac{bB}{m}\cdot\E_{\bepsilon}\sup_{\psi\in \Psi_L,\bw:\norm{\bw}\leq 1}\max\left\{\sum_{i=1}^{m}\epsilon_i \psi\left(\bw^\top\bx_i\right)~,~-\sum_{i=1}^{m}\epsilon_i \psi\left(\bw^\top\bx_i\right)\right\}\notag\\
&\stackrel{(*)}{\leq}~\frac{bB}{m}\cdot\E_{\bepsilon}\left[\sup_{\psi\in \Psi_L,\bw:\norm{\bw}\leq 1}\sum_{i=1}^{m}\epsilon_i \psi\left(\bw^\top\bx_i\right)+\sup_{\psi\in \Psi_L,\bw:\norm{\bw}\leq 1}-\sum_{i=1}^{m}\epsilon_i \psi\left(\bw^\top\bx_i\right)\right]\notag\\
&\stackrel{(**)}{=}~
\frac{2bB}{m}\cdot\E_{\bepsilon}\sup_{\psi\in \Psi_L,\bw:\norm{\bw}\leq 1}\sum_{i=1}^{m}\epsilon_i \psi\left(\bw^\top\bx_i\right)~,\label{eq:frobupbound3}
\end{align}
where $(*)$ follows from the fact that $\max\{a,b\}\leq a+b$ for non-negative $a,b$ and the observation that the supremum is always non-negative (it is only larger, say, than the specific choice of $\psi$ being the zero function), and $(**)$ is by symmetry of the function class $\Psi_L$ (if $\psi\in \Psi_L$, then $-\psi\in \Psi_L$ as well).

Considering \eqref{eq:frobupbound3}, this is $2bB$ times the Rademacher complexity of the function class $\{\bx\mapsto \psi(\bw^\top\bx):\psi\in\Psi_L,\norm{\bw}\leq 1\}$. In other words, this class is a composition of all linear functions of norm at most $1$, and all univariate $L$-Lipschitz functions crossing the origin. Fortunately, the Rademacher complexity of such composed classes was analyzed in \citet{golowich2017size} for a different purpose. In particular, noting that $\bw^\top \bx_i$ is bounded in $[-b_x,b_x]$, and applying Theorem 4 from that paper, we get that \eqref{eq:frobupbound3} is upper bounded by
\begin{equation}
2bB\cdot cL\left(\frac{b_x}{\sqrt{m}}+\log^{3/2}(m)\cdot \Rcal_m(\Hcal)\right)\label{eq:frobupbound4}
\end{equation}
for some universal constant $c>0$, where $\Hcal=\{\bx\mapsto \bw^\top \bx:\norm{\bw}\leq 1\}$, and  $\Rcal_m(\Hcal)$ is the Rademacher complexity of $\Hcal$. 

To complete the proof, we need to employ a standard upper bound on  $\hat{\Rcal}_m(\Hcal)$ (see \citet{bartlett2002rademacher,shalev2014understanding}), which we derive below for completeness:
\begin{align*}
	\hat{\Rcal}_m(\Hcal)~&=~\E_{\bepsilon}\sup_{h\in\Hcal}\frac{1}{m}\sum_{i=1}^{m}\epsilon_i h(\bx_i)~=~\frac{1}{m}\E_{\bepsilon}\sup_{\bw:\norm{\bw}\leq 1}\sum_{i=1}^{m}\epsilon_i \bw^\top \bx_i\\
	&=~\frac{1}{m}\E_{\bepsilon} \sup_{\bw:\norm{\bw}\leq 1}\bw^\top\left(\sum_{i=1}^{m} \epsilon_i \bx_i\right)
	~\stackrel{(*)}{=}~
	\frac{1}{m}\E_{\bepsilon}\left\|\sum_{i=1}^{m}\epsilon_i \bx_i\right\|\\
	&\stackrel{(**)}{\leq}~ \frac{1}{m}\sqrt{\E_{\bepsilon}\left\|\sum_{i=1}^{m}\epsilon_i \bx_i\right\|^2}~=~
	\frac{1}{m}\sqrt{\E_{\bepsilon}\sum_{i,i'=1}^{m}\epsilon_i \epsilon_{i'} \bx_i^\top \bx_{i'}}\\
	&=~ \frac{1}{m}\sqrt{\sum_{i=1}^{m}\norm{\bx_i}^2}~\leq~ \frac{1}{m}\sqrt{mb_x^2}~=~\frac{b_x}{\sqrt{m}}~,
\end{align*}
where $(*)$ is by the Cauchy-Schwarz inequality, and $(**)$ is by Jensen's inequality. Plugging this back into \eqref{eq:frobupbound4}, we get the following upper bound:
\[
2bB\cdot cL\left(\frac{b_x}{\sqrt{m}}+\log^{3/2}(m)\cdot \frac{b_x}{\sqrt{m}}\right)~=~ 2cbBb_x L\cdot \frac{1+\log^{3/2}(m)}{\sqrt{m}}~.
\]
Upper bounding this by $\epsilon$, solving for $m$ and simplifying a bit, the result follows.

\subsection{Proof of \thmref{thm:lowerbound_refined_frobenius}}

We fix a number of inputs $m$ to be chosen later. 
We let $X$ be the $d\times m$ matrix whose $i$-th column is $\bx_i$. We choose X to be any matrix such that the following conditions hold for some universal constant $c>0$:
\begin{itemize}
	\item Every entry of $X$ is in $\{\pm \frac{b_x}{\sqrt{d}}\}$ (hence $\forall i,~\norm{\bx_i}=1$)
	\item $\max_{i'\neq i}|\bx_i^\top \bx_{i'}|\leq c b_x^2\sqrt{\frac{\log(d)}{d}}$
	\item $\norm{X}\leq c b_x\left(1+\sqrt{\frac{m}{d}}\right)$.
\end{itemize}
The existence of such a matrix follows from the probabilistic method: If we simply choose each entry of $X$ independently and uniformly from $\{\pm \frac{1}{\sqrt{d}}\}$, then the first condition automatically holds; The second condition holds with high probability by a standard concentration of measure argument and a union bound; And the third condition holds with arbitrarily high constant probability (by Markov's inequality and the fact that $\E[\norm{\frac{\sqrt{d}}{b_x}\cdot X}]\leq c(\sqrt{d}+\sqrt{m})$, see for example \citet{seginer2000expected}). Thus, by a union bound, a random matrix satisfies all of the above with some positive probability, hence such a matrix $X$ exists. 

Given this input set, it is enough to show that for any $\by\in \{0,1\}^m$, we can find a predictor (namely, $\bu,W$ depending on $\by$) in our class, such that $\norm{\bu}\leq b$, $\norm{W}\leq B$, and 
\begin{equation}\label{eq:shattering_refined}
	\forall i~, \bu^\top \sigma(W\bx_i)~\text{is}~\begin{cases} \leq 0 & y_i=0\\
		\geq 2\epsilon& y_i=1\end{cases}~.
\end{equation}
We will do so as follows: Letting $a\geq 0, p\in [0,1]$ be some parameters to be chosen later, we let
\[
\bu=\frac{b}{\sqrt{n}}\mathbf{1}~~~\text{and}~~~
W = \frac{1}{b_x^2}\cdot V\text{diag}(\by)X^\top~,
\]
Where $V\in \reals^{n\times m}$ is a random matrix with i.i.d. entries chosen as follows:
\[
V_{k,i} ~=~ \begin{cases} 0 & \text{w.p.}~ 1-p\\ a& \text{w.p.}~ \frac{p}{2}\\-a&\text{w.p.}~ \frac{p}{2}~.\end{cases}
\]
Note that the condition $\norm{\bu}\leq b$ follows directly from the definition of $\bu$. We will show that there is a way to choose the parameters $a,p$ such that the following holds: For any $\by\in \{0,1\}^m$, with high probability over the choice of $V$, \eqref{eq:shattering_refined} holds as well as $\norm{W}\leq B$. This implies that for any $\by$, there exists some fixed choice of $V$ (and hence $W$) such that $\norm{W}\leq B$ as well as \eqref{eq:shattering_refined} holds, implying the theorem statement. 

We break this argument into two lemmas:
\begin{lemma}\label{lem:refined_shatter}
	There exists a universal constant $c'>0$ such that the following holds: For any $\epsilon\geq 0$, $\delta\in (0,\exp(-1))$ and $\by\in \{0,1\}^m$, if we assume
	\[
	\beta =  c' a \sqrt{\frac{\log(d)}{d}}\log\left(\frac{m}{\delta}\right)\left(\sqrt{pm}+1\right)
	\]
	as well as $a\geq 4\beta$ and $bap\sqrt{n}\geq 8\epsilon$, then \eqref{eq:shattering_refined} holds with probability at least $1-\delta-m\exp(-pn/16)$ over the choice of $V$. 
\end{lemma}
\begin{proof}
	Let $\bw_k$ be the $k$-th row of $W$. Fixing some $i\in [m]$, we have
	\begin{align}
		\bu^\top \sigma(W\bx_i)~&=~ \bu^\top [W\bx_i-\beta]_+ ~=~ \frac{b}{\sqrt{n}}\sum_{k=1}^{n}[\bw_k^\top \bx_i-\beta]_+
		~=~\frac{b}{\sqrt{n}}\sum_{k=1}^{n}\left[\sum_{i'=1}^{m}\frac{1}{b_x^2}V_{k,i'}y_{i'} \bx_{i'}^\top \bx_i-\beta\right]_+\notag\\
		&=~ \frac{b}{\sqrt{n}}\sum_{k=1}^{n}\left[V_{k,i}y_i+\sum_{i'\neq i}\frac{1}{b_x^2}V_{k,i'}y_{i'}\bx_{i'}^\top \bx_i-\beta\right]_+~.
		\label{eq:shatter_refined}
	\end{align}
	Recalling the assumptions on $X$ and the random choice of $V$, note that $\sum_{i'\neq i}\frac{1}{b_x^2}V_{k,i'}y_{i'}\bx_{i'}^\top \bx_i$ is the sum of $m-1$ independent random variables, each with mean $0$, absolute value at most $|\frac{a}{b_x^2} y_{i'}\bx_{i'^\top}\bx_i|\leq ac\sqrt{\frac{\log(d)}{d}}$, and standard deviation at most $\sqrt{p}\cdot ac\sqrt{\frac{\log(d)}{d}}$. Thus, by Bernstein's inequality, for any $\delta\in (0,\exp(-1))$, it holds with probability at least $1-\delta$ that
	\begin{align*}
		\left|\sum_{i'\neq i}\frac{1}{b_x^2}V_{k,i'}y_{i'}\bx_{i'}^\top \bx_i\right|~&\leq~ c'\left(\sqrt{p}\cdot a\sqrt{\frac{\log(d)}{d}}\cdot\sqrt{(m-1)\log\left(\frac{1}{\delta}\right)}+a\sqrt{\frac{\log(d)}{d}}\cdot\log\left(\frac{1}{\delta}\right)\right)\\
		&\leq~ c' a \sqrt{\frac{\log(d)}{d}}\log\left(\frac{1}{\delta}\right)\left(\sqrt{pm}+1\right)~,
	\end{align*}
	where $c'>0$ is some universal constant. Applying a union bound over all $i\in [m]$, we get that with probability at least $1-\delta$,
	\[
	\max_{i\in [m]}\left|\sum_{i'\neq i}\frac{1}{b_x^2}V_{k,i'}y_{i'}\bx_{i'}^\top \bx_i\right|~\leq~ c'a \sqrt{\frac{\log(d)}{d}}\log\left(\frac{m}{\delta}\right)\left(\sqrt{pm}+1\right)~.
	\]
	Recalling that we choose $\beta$ to equal this upper bound, and plugging back into \eqref{eq:shatter_refined}, we get that with probability at least $1-\delta$, 
	\[
	\forall i\in [m],~~\bu^\top \sigma(W\bx_i)~~\text{is}~~ \begin{cases} \leq \frac{b}{\sqrt{n}}\sum_{k=1}^{n}[V_{k,i} y_i]_+=0& \text{if}~y_i=0\\
		\geq \frac{b}{\sqrt{n}}\sum_{k=1}^{n}[V_{k,i} y_i-2\beta]_+~=~ \frac{b}{\sqrt{n}}\sum_{k=1}^{n}[V_{k,i}-2\beta]_+& \text{if}~y_i=1\end{cases}~.
	\]
	Moreover, by the assumption $a\geq 4\beta$, we have
	\[
	\frac{b}{\sqrt{n}}\sum_{k=1}^{n}[V_{k,i}-2\beta]_+~\geq~\frac{b}{\sqrt{n}}\sum_{k:V_{k,i}=a}\left[a-\frac{a}{2}\right]_+~\geq~ \frac{ba}{2\sqrt{n}}\sum_{k:V_{k,i}=a}1~.
	\]
	Note that $\E_{V}[\sum_{k:V_{k,i}=a}1]=\frac{pn}{2}$. Thus, by a standard multiplicative Chernoff bound and a union bound, $\sum_{k:V_{k,i}=a}1 \geq \frac{pn}{4}$ simultaneously for all $i\in [m]$, with probability at least $1-m\exp(-pn/16)$. Combining with the above using a union bound, we get that with probability at least $1-\delta-m\exp(-pn/16)$ over the choice of $V$, 
	\[
	\forall i\in [m],~~\bu^\top \sigma(W\bx_i)~~\text{is}~~ \begin{cases} \leq 0& \text{if}~y_i=0\\
		\geq \frac{bap\sqrt{n}}{4}& \text{if}~y_i=1\end{cases}~.
	\]
	Since we assume $\frac{bap\sqrt{n}}{4}\geq 2\epsilon$, the result follows. 
\end{proof}

\begin{lemma}\label{lem:frobnormbound_refined}
	For any $\by\in \{0,1\}^m$, with probability at least $\frac{1}{2}$ over the random choice of $V$, the matrix $W$ satisfies
	\[
	\norm{W}_F~\leq~ \frac{a\sqrt{2nmp}}{b_x}~.
	\]
\end{lemma}
\begin{proof}
	By definition of $W,V$ and $X$, we have
	\begin{align*}
		\E[\norm{W}_F^2]~&=~\sum_{k=1}^{n}\sum_{i=1}^{d}\E[W_{k,i}^2]~=~ \sum_{k=1}^{n}\sum_{i=1}^{d}\E\left[\left(\sum_{j=1}^{m}\frac{1}{b_x^2}V_{k,j}y_j X_{i,j}\right)^2\right]\\
		&=~ \frac{1}{b_x^4}\cdot \sum_{k=1}^{n}\sum_{i=1}^{d}\E\left[\sum_{j,j'=1}^{m} V_{k,j}V_{k,j'}y_j y_{j'} X_{i,j} X_{i,j'}\right]\\
		&=~\frac{1}{b_x^4}\cdot\sum_{k=1}^{n}\sum_{i=1}^{d}\sum_{j=1}^{m}\E\left[V_{k,j}^2 y_j^2 X_{i,j}^2\right]~\leq~ \frac{1}{b_x^4}\cdot \frac{b_x^2}{d}\cdot\sum_{k=1}^{n}\sum_{i=1}^{d}\sum_{j=1}^{m}\E[V_{k,j}^2]\\
		&=~ \frac{1}{b_x^2 d}\cdot ndm\cdot pa^2~=~ \frac{nmpa^2}{b_x^2}~.
	\end{align*}
	By Markov's inequality, it follows that with probability at least $\frac{1}{2}$, $\norm{W}_F^2\leq 2\cdot \frac{nmpa^2}{b_x^2}$, from which the result follows. 
\end{proof}

Combining \lemref{lem:refined_shatter} and \lemref{lem:frobnormbound_refined}, and choosing $\delta=1/4$, we get that with some positive probability over the choice of $V$, both the shattering condition in \eqref{eq:shattering_refined} holds, as well as $\norm{W}_F\leq B$, if the following combination of conditions are met (for suitable universal constant $c_1>0$):
\[
m\exp\left(-\frac{pn}{16}\right)< \frac{1}{4}~~,~~
a\geq c_1 a\sqrt{\frac{\log(d)}{d}}\log(4m)(\sqrt{pm}+1)~~,~~ bap\sqrt{n}\geq 8\epsilon~~,~~
a\sqrt{2nmp}~\leq~ Bb_x ~.
\]

We now wish to choose the free parameters $p,a$, to ensure that all these conditions are met (hence we indeed manage to shatter the dataset), while allowing the size $m$ of the shattered set to be as large as possible. We begin by noting that the first condition is satisfied if $p> c_2 \frac{\log(m)}{n}$, and the second condition is satisfied if $d\geq c_3$ and $p\leq c_4\frac{d}{\log(d)\log^2(4m)m}$ (for suitable universal constants $c_2,c_3,c_4>0$). Thus, it is enough to require
\begin{equation}\label{eq:moreconds}
d\geq c_3~~,~~c_2\frac{\log(m)}{n}~<~p~\leq~c_4\frac{d}{\log(d)\log^2(4m)m}~~,~~ bap\sqrt{n}\geq 8\epsilon~~,~~
a\sqrt{2nmp}\leq Bb_x~.
\end{equation}
Let us pick in particular 
\[
p~=~c_4\frac{d}{\log(d)\log^2(4m) m}
\]
(which is valid if it is in $[0,1]$ and if $c_2\frac{\log(m)}{n}\leq c_4\frac{d}{\log(d)\log^2(4m) m}$, or equivalently $m\log(m)\log^2(4m)\leq \frac{c_4 n d}{c_2 \log(d)})$ and
\[
a~=~\frac{8\epsilon}{bp\sqrt{n}}~=~ \frac{8\epsilon \log(d)\log^2(4m)m}{c_4 b d\sqrt{n}}
\]
(which automatically satisfied the third condition in \eqref{eq:moreconds}). Plugging into \eqref{eq:moreconds}, the required conditions hold if we assume
\begin{align*}
	&d\geq c_3~~,~~\frac{c_4 d}{\log(d)\log^2(4m) m}\leq 1~~,~~m\log^3(4m)\leq \frac{c_5 n d}{\log(d)}~~,~~
	c_6 \frac{\epsilon \sqrt{\log(d)}\log(4m) m}{b\sqrt{d}}\leq Bb_x
\end{align*}
for appropriate universal constants $c_5,c_6>0$. The first two conditions are satisfied if we require $m\geq c_7 d \geq c_8$ for suitable universal constants $c_7,c_8>0$. Thus, it is enough to require the set of conditions
\[
m\geq c_6 d \geq c_7~~,~~m\log^3(4m)\leq \frac{c_5 n d}{\log(d)}~~,~~ m\log(4m)\leq \frac{bBb_x\sqrt{d}}{c_6\epsilon \sqrt{\log(d)}}~.
\]
All these conditions are satisfied if we assume $d\geq c_7/c_6$, pick
\begin{equation}\label{eq:malmostfinal}
	m = \tilde{\Theta}\left(\min\left\{nd, \frac{bB b_x}{\epsilon}\sqrt{d}\right\}\right)
\end{equation}
(with the $\tilde{\Theta}$ hiding constants and factors polylogarithmic in $d,n,b,B,b_x,\frac{1}{\epsilon})$), and assume that the parameters are such that this expression is sufficiently larger than $d$, and that $d$ is larger than some universal constant.

It only remains to track what value of $\beta$ we have chosen (as a function of the problem parameters). Combining \lemref{lem:refined_shatter}, the choice of $a,p$ from earlier, as well as \eqref{eq:malmostfinal}, it follows that
\[
\beta ~=~ \tilde{\Theta}\left(\frac{a}{\sqrt{d}}(1+\sqrt{pm})\right)~=~ \tilde{\Theta}\left(\frac{\epsilon m}{bd\sqrt{dn}}(1+\sqrt{d})\right)~=~\tilde{\Theta}\left(\frac{\epsilon m}{bd\sqrt{n}}\right)
~=~ \tilde{\Theta}\left(\min\left\{\frac{\epsilon\sqrt{n}}{b}~,~\frac{Bb_x}{\sqrt{dn}}\right\}\right)~,
\]
which is at most $\tilde{\Ocal}(B b_x/\sqrt{dn})$.

\subsection{Proof of Corollary \ref{cor:lowerbound_refined_frobenius}}

\thmref{thm:lowerbound_refined_frobenius} implies that a certain dataset $\{\bx_i\}_{i=1}^{m}$ of points in $\reals^{d}$ of norm at most $b_x$ (where $m$ is the lower bound stated in the theorem) can be shattered with margin $\epsilon$, using networks in $\Fcal^{\sigma}_{b,B,n,d}$ of the form $\bx\mapsto \bu^\top \sigma(W\bx)$, where $\sigma=[z-\beta]_+$ for some $\beta\in \left[0,\tilde{\Ocal}(\frac{Bb_x}{\sqrt{dn}})\right]$. Our key observation is the following: Any network $\bx\mapsto \bu^\top \sigma(W\bx)$ can be equivalently written as $\tilde{\bx}\mapsto \bu^\top [\tilde{W}\tilde{\bx}]_+$, where $\tilde{\bx}=(\bx,b_x)$, and $\tilde{W} = [W~,~-\frac{\beta}{b_x}\cdot \mathbf{1}]$ (namely, we add to $W$ another column with every entry being equal to $-\frac{\beta}{b_x}$. Note that if $\norm{\bx}\leq b_x$, then $\norm{\tilde{\bx}}\leq \sqrt{2}b_x$, and $\norm{\tilde{W}}\leq \norm{W}+\norm{-\frac{\beta}{b_x}\cdot \mathbf{1}}\leq B+\frac{\beta}{b_x}\sqrt{n} \leq 2B$ under the corollary's conditions. Thus, if we can shatter a set of points $\{\bx_i\}_{i=1}^{m}$ in the unit ball in $\reals^d$ using networks from $\Fcal^{\sigma}_{b,B,n,d}$, we can also shatter $\{\tilde{\bx}_i\}_{i=1}^{m}$ in $\reals^{d+1}$ (with norm $\leq \sqrt{2}b_x$) using networks from $\Fcal^{[\cdot]_+}_{b,2B,n,d+1}$. Rescaling $b_x,B,d$ appropriately, we get the same shattering number lower bound for $\Fcal^{[\cdot]_+}_{b,B,n,d}$ and inputs with norm $\leq b_x$ up to small numerical constants which get absorbed into the $\tilde{\Omega}(\cdot)$ notation.

\subsection{Proofs of \thmref{thm:smooth_upper} and \thmref{thm:smooth_upper_deep}}

In what follows, given a vector $\bu_i$, we let $u_{i,j}$ denote its $j$-th entry. 

The proofs rely on the following two key technical lemmas:

\begin{lemma}\label{lem:reduction}
	Let $W$ be a matrix such that $\norm{W}\leq 1$, with row vectors $\bw_1,\bw_2,\ldots$ Then the following holds for any set of vectors $\{\bu_i\}$ with the same dimensionality as $\bw_1$, and any scalars $\{z_{i,\ell}\},\{z_{i}\}$indexed by $i,\ell$:
	\[
	\sum_{\ell}\left(\sum_i (\bw_\ell^\top \bu_i)z_{i,\ell}\right)^2~\leq~ \sum_{\ell,r}\left(\sum_{i}u_{i,r}z_{i,\ell}\right)^2
	\]
	and
	\[
	\sum_{\ell}\left(\sum_i (\bw_{\ell}^\top \bu_i)z_i\right)^2 ~\leq~ \sum_{r}\left(\sum_i u_{i,r}z_i\right)^2~,
	\]
	where the sum $r$ is over all all coordinates of each $\bu_i$.
\end{lemma}
\begin{proof}
	Starting with the first inequality, the left hand side equals
	\[
	\sum_\ell\left(\bw_\ell^\top\left(\sum_i \bu_i z_{i,\ell}\right)\right)^2~\leq~ \sum_{\ell,\ell'}\left(\bw_{\ell'}^\top\left(\sum_i \bu_i z_{i,\ell}\right)\right)^2~=~ \sum_\ell \left\|W^\top \left(\sum_i \bu_i z_{i,\ell}\right)\right\|^2~.
	\]
	By Cauchy-Schwartz and the assumption $\norm{W}\leq 1$, this is at most $\sum_\ell \left\|\sum_i \bu_i z_{i,\ell}\right\|^2~=~ \sum_{\ell,r} \left(\sum_i u_{i,r} z_{i,\ell}\right)^2$. As to the second inequality, the left hand side equals
	\[
	\sum_{\ell}\left(\bw_{\ell}^\top\left(\sum_i \bu_i z_{i}\right)\right)^2 ~=~ \left\|W^\top \left(\sum_i \bu_i z_{i}\right)\right\|^2~\leq~ \left\|\sum_i \bu_i z_{i}\right\|^2 = \sum_{r}\left(\sum_i u_{i,r}z_i\right)^2
	\]
	where we again used Cauchy Schwartz and the assumption $\norm{W}\leq 1$. 
\end{proof}

\begin{lemma}\label{lem:layerredux}
	Given a vector $\bu\in \reals^{d_{in}}$, a matrix $W\in \reals^{d_{out}\times d_{in}}$ with row vectors $\bw_1,\bw_2,\ldots$ such that $\norm{W}\leq B$, and a positive integer $k$, define 
	\[
	f(\bu) = (W \bu)^{\circ k}~,
	\]
	where $^{\circ k}$ denotes taking the $k$-th power element-wise. Then for any positive integer $r$, any vectors $\bu_1,\bu_2,\ldots$ in $\reals^{d_{in}}$ and any scalars $\epsilon_1,\epsilon_2,\ldots$, it holds that
	\[
	\sum_{\ell_1,\ldots,\ell_r=1}^{d_{out}}\left(\sum_i \epsilon_i f(\bu_i)_{\ell_1}\cdots f(\bu_i)_{\ell_r}\right)^2
	~\leq~
	B^{2rk}\cdot \sum_{\ell_1,\ldots,\ell_{rk}=1}^{d_{in}}\left(\sum_i \epsilon_i u_{i,\ell_1}\cdots u_{i,\ell_{rk}}\right)^2~.
	\] 
\end{lemma}
\begin{proof}
	It is enough to prove the result for $W$ such that $\norm{W}=1$ (and therefore $B=1$): For any other $W$, apply the result on $\tilde{f}(\bu) := (\frac{W}{\norm{W}}\bu)^{\circ k} = \frac{1}{\norm{W}^k} f(\bu)$, and rescale accordingly. 
	
	The left hand side equals
	\begin{equation}\label{eq:redux1}
		\sum_{\ell_1\ldots \ell_r=1}^{d_{out}}\left(\sum_i \epsilon_i (\bw_{\ell_1}^\top \bu_i)^{\circ k}\cdots (\bw_{\ell_r}^{\top}\bu_i)^{\circ k}\right)^2
	\end{equation}
	Note that the term inside the square involves the product of $rk$ terms. We now simplify them one-by-one using \lemref{lem:reduction}: To start, we note that the above can be written as
	\[
	\sum_{\ell_2\ldots \ell_r=1}^{d_{out}}\sum_{\ell_1=1}^{d_{out}}\left(\sum_i (\bw_{\ell_1}^\top \bu_i)\cdot \epsilon_i (\bw_{\ell_1}^\top\bu_i)^{\circ k-1}(\bw_{\ell_2}^\top \bu_i)^{\circ k}\cdots (\bw_{\ell_r}^\top \bu_i)^{\circ k}\right)^2
	\]
	Denoting $\epsilon_i (\bw_{\ell_1}^\top\bu_i)^{\circ k-1}(\bw_{\ell_2}^\top \bu_i)^{\circ k}\cdots (\bw_{\ell_r}^\top \bu_i)^{\circ k}$ as $z_{i,\ell_1}$ and plugging the first inequality in \lemref{lem:reduction}, the above is at most
	\[
	\sum_{\ell_2\ldots \ell_r=1}^{d_{out}}\sum_{\ell_1=1}^{d_{out}}\sum_{\ell'_1=1}^{d_{in}}
	\left(\sum_i u_{i,\ell'_1}\epsilon_i (\bw_{\ell_1}^\top\bu_i)^{\circ k-1}(\bw_{\ell_2}^\top \bu_i)^{\circ k}\cdots (\bw_{\ell_r}^\top \bu_i)^{\circ k}\right)^2
	\]
	Again pulling out one of the product terms in front, we can rewrite this as
	\[
	\sum_{\ell_2\ldots \ell_r=1}^{d_{out}}\sum_{\ell'_1=1}^{d_{in}}\sum_{\ell_1=1}^{d_{out}}
	\left(\sum_i(\bw_{\ell_1}^\top\bu_i)\cdot u_{i,\ell'_1}\epsilon_i (\bw_{\ell_1}^\top\bu_i)^{\circ k-2}(\bw_{\ell_2}^\top \bu_i)^{\circ k}\cdots (\bw_{\ell_r}^\top \bu_i)^{\circ k}\right)^2~.
	\]
	Again using the first inequality in \lemref{lem:reduction}, this is at most
	\[
	\sum_{\ell_2\ldots \ell_r=1}^{d_{out}}\sum_{\ell'_1,\ell''_1=1}^{d_{in}}\sum_{\ell_1=1}^{d_{out}}
	\left(\sum_i u_{i,\ell''_1} u_{i,\ell'_1}\epsilon_i (\bw_{\ell_1}^\top\bu_i)^{\circ k-2}(\bw_{\ell_2}^\top \bu_i)^{\circ k}\cdots (\bw_{\ell_r}^\top \bu_i)^{\circ k}\right)^2~.
	\]
	Repeating this to get rid of all but the last $(\bw_{\ell_1}^\top \bu_i)$ term, we get the upper bound
	\[
	\sum_{\ell_2\ldots \ell_r=1}^{d_{out}}\sum_{\ell_1^1\ldots \ell_1^{k-1}=1}^{d_{in}}\sum_{\ell_1=1}^{d_{out}}
	\left(\sum_i u_{i,\ell_1^1}\cdots u_{i,\ell_1^{k-1}}\epsilon_i  (\bw_{\ell_1}^\top\bu_i)(\bw_{\ell_2}^\top \bu_i)^{\circ k}\cdots (\bw_{\ell_r}^\top \bu_i)^{\circ k}\right)^2~.
	\]
	Again pulling the last $(\bw_{\ell_1}^\top \bu_i)$ term in front, and applying now the second inequality in \lemref{lem:reduction} (as the remaining terms in the product no longer depend on $\ell_1$), we get the upper bound
	\[
	\sum_{\ell_2\ldots \ell_r=1}^{d_{out}}\sum_{\ell_1^1\ldots \ell_1^{k}=1}^{d_{in}}
	\left(\sum_i u_{i,\ell_1^1}\cdots u_{i,\ell_1^{k}}\epsilon_i  (\bw_{\ell_2}^\top \bu_i)^{\circ k}\cdots (\bw_{\ell_r}^\top \bu_i)^{\circ k}\right)^2~.
	\]
	Recalling that this is an upper bound on \eqref{eq:redux1}, and applying the same procedure now on the $(\bw_{\ell_2}^\top \bu_i),(\bw_{\ell_3}^\top \bu_i)\ldots$ terms, we get overall an upper bound of the form
	\[
	\sum_{\ell_1^1\ldots \ell_1^k=1}^{d_{in}}\cdots \sum_{\ell_r^1\cdots \ell_r^k=1}^{d_{in}}
	\left(\sum_i u_{i,\ell_1^1}\cdots u_{i,\ell_r^k}\epsilon_i \right)^2~. 
	\]
	Re-labeling the $rk$ indices as $\ell_1,\ldots,\ell_{rk}$, the result follows.
\end{proof}

\subsubsection{Proof of \thmref{thm:smooth_upper}}
Fixing a dataset $\bx_1,\ldots,\bx_m$ and applying Cauchy-Schwartz, the Rademacher complexity is
\[
\E_{\bepsilon}\sup_{\bu,W}\frac{1}{m}\sum_{i=1}^{m}\epsilon_i \bu^\top \sigma(W\bx_i)~\leq~ \E_{\bepsilon}\sup_{W} \frac{b}{m}\left\|\sum_{i=1}^{m}\epsilon_i \sigma(W\bx_i)\right\|~.
\]
Recalling that $\sigma(z)=\sum_{j=1}^{\infty}a_j z^j$, by the triangle inequality, we have that the above is at most
\[
\E_{\bepsilon}\sup_{W} \frac{b}{m}\sum_{j=1}^{\infty}|a_j|\left\|\sum_{i=1}^{m}\epsilon_i (W\bx_i)^j\right\|~\leq~
\frac{b}{m}\sum_{j=1}^{\infty}|a_j|\E_{\bepsilon}\sup_W \left\|\sum_{i=1}^{m}\epsilon_i(W\bx_i)^j\right\|
\]
where $(\cdot)^j$ is applied element-wise. Recalling that the supremum is over matrices of spectral norm at most $B$, and using Jensen's inequality, the above can be equivalently written as
\begin{equation}
	\frac{b}{m}\sum_{j=1}^{\infty}|a_j|B^j\cdot\E_{\bepsilon}\sup_{W:\norm{W}\leq 1} \left\|\sum_{i=1}^{m}\epsilon_i(W\bx_i)^j\right\|~\leq~
	\frac{b}{m}\sum_{j=1}^{\infty}|a_j|B^j\sqrt{\E_{\bepsilon}\sup_{W:\norm{W}\leq 1} \left\|\sum_{i=1}^{m}\epsilon_i(W\bx_i)^j\right\|^2}~.\label{eq:zeroj}
\end{equation}
Using \lemref{lem:layerredux}, we have that for any $W:\norm{W}\leq 1$,
\[
\left\|\sum_{i=1}^{m}\epsilon_i(W\bx_i)^j\right\|^2
~=~ \sum_{\ell}\left(\sum_i \epsilon_i (W \bx_i)_{\ell}^j\right)^2
~\leq~\sum_{\ell_1,\ldots,\ell_j=1}^{d}\left(\sum_{i=1}^{m}\epsilon_i x_{i,\ell_1}\cdots x_{i,\ell_j}\right)^2~.
\]
Thus, 
\begin{align*}
	&\E_{\bepsilon}\sup_{W:\norm{W}\leq 1} \left\|\sum_{i=1}^{m}\epsilon_i(W\bx_i)^j\right\|^2
	~\leq~ \E_{\bepsilon} \sum_{\ell_1,\ldots,\ell_j=1}^{d}\left(\sum_{i=1}^{m}\epsilon_i x_{i,\ell_1}\cdots x_{i,\ell_j}\right)^2\\
	&=~ \E_{\bepsilon}\sum_{i,i'=1}^{m}\sum_{\ell_1,\ldots,\ell_j=1}^{d} \epsilon_i \epsilon_{i'}x_{i,\ell_1}x_{i',\ell_1}\cdots x_{i,\ell_j}x_{i',\ell_j}\\
	&\stackrel{(*)}{=}~\sum_{i=1}^{m}\sum_{\ell_1,\ldots,\ell_j=1}^{d}x_{i,\ell_1}^2\cdots x_{i,\ell_j}^2\\
	&=~\sum_{i=1}^{m}\left(\sum_{\ell_1=1}^{d}x_{i,\ell_1}^2\right)\cdots\left(\sum_{\ell_j=1}^{d}x_{i,\ell_j}^2\right)\\
	&=~\sum_{i=1}^{m}\norm{\bx_i}^{2j}~\leq~ \sum_{i=1}^{m} b_x^{2j} ~=~ m\cdot b_x^{2j}~,
\end{align*}
where in $(*)$ we used the fact that each $\epsilon_i$ is independent and uniformly distributed on $\pm 1$. Plugging this bound back into \eqref{eq:zeroj}, we get that the Rademacher complexity is at most
\[
\frac{b}{m}\sum_{j=1}^{\infty}|a_j| (Bb_x)^j \sqrt{m}~=~ \frac{b\cdot \tilde{\sigma}(B b_x)}{\sqrt{m}}~.
\]
Upper bounding this by $\epsilon$ and solving for $m$, the result follows.

\subsection{Proof of \thmref{thm:smooth_upper_deep}}

For simplicity, we use $\sup_{\bu,W^1,\ldots,W^{L}}$ as short for $\sup_{\bu:\norm{\bu}\leq b, W^1,\ldots,W^{L}:\max_j \norm{W^j}\leq B}$. 
The Rademacher complexity equals
\begin{align}
	&\E_{\bepsilon} \sup_{\bu,W^1,\ldots,W^{L}} \frac{1}{m}\sum_{i=1}^{m}\epsilon_i f_{L+1}(\bx_i)
	~=~
	\E_{\bepsilon} \sup_{\bu,W^1,\ldots,W^{L}} \frac{1}{m}\sum_{i=1}^{m}\epsilon_i \bu^\top f_{L}(\bx_i)\notag\\
	&\leq~
	\E_{\bepsilon} \sup_{\bu,W^1,\ldots,W^{L}} \bu^\top\left(\frac{1}{m}\sum_{i=1}^{m}\epsilon_i f_{L}(\bx_i)\right)
	~\leq~
	\frac{b}{m}\cdot \E_{\bepsilon} \sup_{\bu,W^1,\ldots,W^{L}} \left\|\sum_{i=1}^{m}\epsilon_i  f_{L}(\bx_i)\right\|\notag\\
	&\leq~ \frac{b}{m}\sqrt{ \E_{\bepsilon} \sup_{\bu,W^1,\ldots,W^{L}} \left\|\sum_{i=1}^{m}\epsilon_i  f_{L}(\bx_i)\right\|^2}~=~
	\frac{b}{m}\sqrt{ \E_{\bepsilon} \sup_{\bu,W^1,\ldots,W^{L}} \sum_{\ell}\left(\sum_{i=1}^{m}\epsilon_i  (f_{L}(\bx_i))_{\ell}\right)^2}~,\label{eq:reduxL}
\end{align}
where we used Cauchy-Schwartz and the assumption $\norm{\bu}\leq b$, and $\ell$ ranges over the indices of $f_{L}(\bx_i)$. Recalling that $f_{j+1}(\bx) = (W^{j+1} f_j(\bx))^{\circ k}$ and repeatedly applying \lemref{lem:layerredux}, we have
\begin{align*}
	&\sum_{\ell} \left(\sum_{i=1}^{m}\epsilon_i (f_{L}(\bx_i))_{\ell}\right)^2~\leq~\sum_{\ell} B^{2k}\sum_{\ell_1\ldots\ell_k}\left(\sum_{i=1}^{m}\epsilon_i (f_{L-1}(\bx_i))_{\ell_1}\cdots (f_{L-1}(\bx_i))_{\ell_k}\right)^2\\
	&\leq~
	B^{2k+2k^2}\sum_{\ell_1\ldots \ell_{k^2}}\left(\sum_{i=1}^{m}\epsilon_i (f_{L-2}(\bx_i))_{\ell_1}\cdots (f_{L-2}(\bx_i))_{\ell_k}\right)^2\\
	&\leq~\cdots~\leq~
	B^{2k+2k^2+\ldots 2k^{L}}\sum_{\ell_1\ldots \ell_{k^{L}}}\left(\sum_{i=1}^{m}\epsilon_i (f_0(\bx_i))_{\ell_1}\cdots (f_0(\bx_i))_{\ell_{k^{L}}}\right)^2\\
	&=~
	B^{2k+2k^2+\ldots 2k^{L}}\sum_{\ell_1\ldots \ell_{k^{L}}}\left(\sum_{i=1}^{m}\epsilon_i (f_0(\bx_i))_{\ell_1}\cdots (f_0(\bx_i))_{\ell_{k^{L}}}\right)^2\\
	&=~
	B^{2k+2k^2+\ldots 2k^{L}}\sum_{\ell_1\ldots \ell_{k^{L}}}\left(\sum_{i=1}^{m}\epsilon_i x_{i,\ell_1}\cdots x_{i,\ell_{k^{L}}}\right)^2\\
\end{align*}
Therefore, recalling that $\epsilon_1\ldots \epsilon_m$ are i.i.d. and uniform on $\{-1,+1\}$, we have
\begin{align*}
	&\E_{\bepsilon} \sup_{\bu,W^0,\ldots,W^{L-1}} \sum_{\ell}\left(\sum_{i=1}^{m}\epsilon_i  (f_{L}(\bx_i))_{\ell}\right)^2
	~\leq~ 
	B^{2k+2k^2+\ldots 2k^{L}}\E_{\bepsilon}\sum_{\ell_1\ldots \ell_{k^{L}}}\left(\sum_{i=1}^{m}\epsilon_i x_{i,\ell_1}\cdots x_{i,\ell_{k^{L}}}\right)^2\\
	&=~B^{2k+2k^2+\ldots 2k^{L}}\E_{\bepsilon}\sum_{\ell_1\ldots \ell_{k^{L}}}\sum_{i,i'=1}^{m}\epsilon_i\epsilon_{i'} x_{i,\ell_1}x_{i',\ell_1}\cdots x_{i,\ell_{k^{L}}}x_{i',\ell_{k^{L}}}\\
	&=~ B^{2k+2k^2+\ldots 2k^{L}}\sum_{\ell_1\ldots \ell_{k^{L}}}\sum_{i=1}^{m} x_{i,\ell_1}^2\cdots x_{i,\ell_{k^{L}}}^2\\
	&=~B^{2k+2k^2+\ldots 2k^{L}}\sum_{i=1}^{m}\left(\sum_{\ell_1} x_{i,\ell_1}^2\right)\cdots \left(\sum_{\ell_{k^{L}}}x_{i,\ell_{k^{L}}}^2\right)~\leq~B^{2k+2k^2+\ldots 2k^{L}}\cdot m\cdot b_x^{2k^{L}}~,
\end{align*}
where in the last step we used the assumption that $\norm{\bx_i}^2\leq b_x^2$ for all $i$. Plugging this back into \eqref{eq:reduxL}, and solving for the number of inputs $m$ required to make the expression less than $\epsilon$, the result follows.

\subsection{Proof of \thmref{thm:convupbound}}

We will need the following lemma, based on a contraction result from \citet{ledoux1991probability}:
\begin{lemma}\label{lem:ledoux}
    Let $\Tcal$ be a set of vectors in $\reals^m$ which contains the origin. 
    If $\epsilon_1,\ldots,\epsilon_m$ are i.i.d. Rademacher random variables, and $\sigma$ is an $L$-Lipschitz function on $\reals$ with $\sigma(0)=0$, then
    \[
    \E_{\bepsilon}\left[\sup_{t\in \Tcal}\left(\sum_{i=1}^{m}\epsilon_i \sigma(t_i)\right)^2\right]~\leq~ 2L^2\cdot \E_{\bepsilon}\left[ \left(\sup_{t\in \Tcal}\sum_{i=1}^{m}\epsilon_i t_i\right)^2\right]~.
    \]
\end{lemma}
\begin{proof}
    For any realization of $\bepsilon$, $\sup_{t\in \Tcal}|\sum_{i=1}^{m}\epsilon_i \sigma(t_i)|$ equals either $\sup_{t\in \Tcal}\sum_{i=1}^{m}\epsilon_i \sigma(t_i)$ or $\sup_{t\in \Tcal}-\sum_{i=1}^{m}\epsilon_i \sigma(t_i)$. Thus, the left hand side in the lemma can be upper bounded as follows:
    \[
        \E\left[\left(\sup_{t\in \Tcal} \left|\sum_{i=1}^{m}\epsilon_i \sigma(t_i)\right|\right)^2\right]~\leq~\E\left[\left(\sup_{t\in \Tcal} \sum_{i=1}^{m}\epsilon_i \sigma(t_i)\right)^2+
        \left(\sup_{t\in \Tcal}-\sum_{i=1}^{m}\epsilon_i \sigma(t_i)\right)^2\right]~.
    \]
    Noting that $\E_{\bepsilon}[(\sup_{t\in \Tcal}\sum_i \epsilon_i \sigma(t_i))^2]$ equals $\E_{\bepsilon}[(\sup_{t\in \Tcal}-\sum_i \epsilon_i \sigma(t_i))^2]$ by symmetry of the $\epsilon_i$ random variables, the expression above equals
    \[
        2\cdot \E\left[\left(\sup_{t\in \Tcal}\sum_{i=1}^{m}\epsilon_i \sigma(t_i)\right)^2\right]
        ~\stackrel{(*)}{=}~ 2\cdot \E\left[\left[\sup_{t\in \Tcal} \sum_{i=1}^{m}\epsilon_i \sigma(t_i)\right]_+^2\right]
        ~=~2L^2\cdot \E\left[\left[\sup_{t\in \Tcal}\sum_{i=1}^{m}\epsilon_i \frac{1}{L}\sigma(t_i)\right]_+^2\right]~,
    \]
    where $(*)$ follows from the fact that the supremum is always non-negative, since $\sigma(0)=0$ and $\Tcal$ contains the origin. We now utilize equation (4.20) in \cite{ledoux1991probability}, which implies that $\E_{\bepsilon} g(\sup_{t\in \Tcal}\sum_i \epsilon_i \phi(t_i))\leq \E_{\bepsilon} g(\sup_{t\in \Tcal}\sum_i \epsilon_i t_i)$ for any $1$-Lipschitz $\phi$ satisfying $\phi(0)=0$, and any convex increasing function $g$. Plugging into the above, and using the fact that $[z]_+^2\leq z^2$ for all $z$, the lemma follows.
\end{proof}

We now turn to prove the theorem. The Rademacher complexity times $m$ equals
\[    
\E_{\bepsilon}\left[\sup_{W,\bu} \sum_{i=1}^{m}\epsilon_i \bu^\top \sigma(W\bx_i)\right]~,
\]
where for notational convenience we drop the conditions on $W,\bu,\bw$ in the supremum. Using the Cauchy-Schwartz and Jensen's inequalities, this in turn can be upper bounded as follows:
\begin{align*}
    \E_{\bepsilon}&\left[\sup_{W,\bu}\bu^\top\left(\sum_{i=1}^{m}\epsilon_i \sigma(W\bx_i)\right)\right]~\leq~ b\cdot 
    \E_{\bepsilon}\left[\sup_{W}\left\|\sum_{i=1}^{m}\epsilon_i \sigma(W\bx_i)\right\|\right]\\
    &\leq~ b\sqrt{\E_{\bepsilon}\left[\sup_{W}\left\|\sum_{i=1}^{m}\epsilon_i \sigma(W\bx_i)\right\|^2\right]}~=~ b\sqrt{\E_{\bepsilon}\left[\sup_W \sum_{j=1}^{n}\left(\sum_{i=1}^{m}\epsilon_i \sigma(\bw^\top \phi_j(\bx_i))\right)^2\right]}\\
    &\leq~ b\sqrt{\sum_{j=1}^{n}\E_{\bepsilon}\left[\sup_W \left(\sum_{i=1}^{m}\epsilon_i \sigma(\bw^\top \phi_j(\bx_i))\right)^2\right]}~.
\end{align*}
Recall that the supremum is over all matrices $W$ which conform to the patches, and has spectral norm at most $B$. By definition, every row of this matrix has a subset of entries, which correspond to the convolutional filter vector $\bw$. Thus, we must have $\norm{\bw}\leq B$, since the norm $\bw$ equals the norm of any row of $W$, and the norm of a row of $W$ is a lower bound on the spectral norm. Thus, we can upper bound the expression above by taking the supremum over \emph{all} vectors $\bw$ such that $\norm{\bw}\leq B$ (and not just those that the corresponding matrix has spectral norm $\leq B$). Thus, we get the upper bound
\[
b\sqrt{\sum_{j=1}^{n}\E_{\bepsilon}\left[\sup_{\bw:\norm{\bw}\leq B} \left(\sum_{i=1}^{m}\epsilon_i \sigma(\bw^\top \phi_j(\bx_i))\right)^2\right]},
\]
which by \lemref{lem:ledoux} and Cauchy-Shwartz, is at most
\begin{align*}
bL&\sqrt{2\sum_{j=1}^{n}\E_{\bepsilon}\left[\sup_{\bw:\norm{\bw}\leq B} \left(\sum_{i=1}^{m}\epsilon_i \bw^\top \phi_j(\bx_i))\right)^2\right]}~\leq~ bBL
\sqrt{2\sum_{j=1}^{n}\E_{\bepsilon}\left[\left\|\sum_{i=1}^{m}\epsilon_i \phi_j(\bx_i))\right\|^2\right]}\\
&=~ bBL
\sqrt{2\sum_{j=1}^{n}\E_{\bepsilon}\left[\sum_{i,i'=1}^{m}\epsilon_i \epsilon_i' \phi_j(\bx_i)^\top \phi_j(\bx_{i'})\right]}
~=~
bBL
\sqrt{2\sum_{j=1}^{n}\sum_{i=1}^{m}\norm{\phi_j(\bx_i)}^2}.
\end{align*}
Recalling that $O_{\Phi}$ is the maximal number of times any single input coordinate appears across the patches, and letting $x_{i,l}$ be the $l$-th coordinate of $\bx_i$, we can upper bound the above by
\[
bBL
\sqrt{2\sum_{i=1}^{m}\sum_{l=1}^{d}x_{i,l}^2 O_{\Phi}}~=~
bBL\sqrt{2\sum_{i=1}^{m}\norm{\bx_i}^2\cdot O_{\Phi}}~\leq~ bBb_x L \sqrt{2m O_{\Phi}}. 
\]
Dividing by $m$, and solving for the number $m$ required to make the resulting expression less than $\epsilon$, the result follows.

\subsection{Proof of  \thmref{thm:CNN_upper_bound}}

The proof follows from a covering number argument. We start with some required definitions and lemmas.

\begin{definition}
	Let $\Fcal$ be a class of functions from $\Xcal$ to $\reals$.  For $1 \leq p \leq \infty$, $\epsilon>0$, and $\{\bx_1,\ldots,\bx_m\} \subseteq \Xcal$, the \emph{empirical covering number} $\Ncal_p(\Fcal,\epsilon; \bx_1,\ldots,\bx_m)$ is the minimal cardinality of a set $V \subseteq \reals^m$, such that for all $f \in \Fcal$ there is $\bv \in V$ such that 
	\[
		\left( \frac{1}{m}\sum_{i=1}^m | f(\bx_i) - v_i |^p \right)^{1/p} \leq \epsilon~.
	\]
	We define the \emph{covering number} $\Ncal_p(\Fcal, \epsilon, m) = \sup_{\bx_1,\ldots,\bx_m} \Ncal_p(\Fcal,\epsilon; \bx_1,\ldots,\bx_m)$.
\end{definition}

\begin{lemma}[\cite{zhang2002covering}]\label{lem:from_zhang}
	Let $a,b>0$, let $\Xcal = \{\bx \in \reals^d~:~\norm{\bx} \leq b\}$, and consider the class of linear predictors $\Fcal = \{f \in \reals^\Xcal:~f(\bx) = \bw^\top \bx,~\norm{\bw} \leq a \}$. Then,
	\[
		\log \Ncal_\infty(\Fcal,\epsilon,m) 
		\leq \frac{36a^2b^2}{\epsilon^2} \log\left( 2m \lceil 4ab/\epsilon +2 \rceil +1 \right)~.
	\]
\end{lemma}

\begin{lemma}[E.g., \cite{daniely2019generalization}]\label{lem:from_daniely}
	Let $C>0$ and let $\Fcal$ be a class of $C$-bounded functions from $\Xcal$ to $\reals$, i.e., 
	$| f(\bx) | \leq C$ for all $f \in \Fcal$ and $\bx \in \Xcal$.
	Then, for every integer $M \geq 1$ we have
	\[
		\Rcal_m(\Fcal) \leq C 2^{-M} + \frac{6C}{\sqrt{m}} \sum_{k=1}^M 2^{-k} \sqrt{\log \Ncal_2(\Fcal, C2^{-k}, m)}~.
	\]
\end{lemma}


We are now ready to prove the theorem.
For $i \in [m],~j \in [n]$ we denote $\bx'_{i,j} =  \phi_j(\bx_i) \in \reals^{n'}$.
Let $\Xcal_{n'} = \{\bx' \in \reals^{n'}: \norm{\bx'} \leq b_x\}$, and let 
\[
	\Fcal := \{f \in \reals^{\Xcal_{n'}}~:~f(\bx') = \bw^\top \bx',~\bw \in \reals^{n'},~\norm{\bw} \leq B \}~.
\]
Let $V \subseteq \reals^{mn}$ be a set 
of size at most $\Ncal_\infty(\Fcal,\epsilon / L,mn)$, such that for all $f \in \Fcal$ there is $\bv \in V$ that satisfies the following: Letting $v_{i,j} := v_{(i-1)n+j}$, we have $| f(\bx'_{i,j}) - v_{i,j} | \leq \epsilon / L$ for all $i \in [m],~j \in [n]$.

We define
\[
	U := \{\bu \in \reals^m~:~\text{there is } \bv \in V \text{ s.t. } u_i = \rho \circ \sigma (v_{i,1},\ldots,v_{i,n}) = \rho\left( \sigma(v_{i,1}), \ldots, \sigma(v_{i,n})\right) \text{ for all } i \in [m] \}~.\] 
	Note that $|U| \leq |V|$. 
	Let $h \in  \Hcal_{B,n,d}^{\sigma, \rho, \Phi}$ and suppose that the network $h$ has a filter $\bw \in \reals^{n'}$. Let $W$ be the weight matrix that corresponds to $\Phi$ and $\bw$. Thus, we have $\norm{W} \leq B$. Let $\bx \in \reals^d$ such that $\phi_1(\bx)=\frac{\bw}{\norm{\bw}}$ and $x_k=0$ for every coordinate $k$ that does not appear in $\phi_1$. That is, $\bx$ is a vector of norm $1$ such that $(W \bx)_1 = \bw^\top \phi_1(\bx) = \norm{\bw}$. Therefore, $\norm{W \bx} \geq (W \bx)_1 =  \norm{\bw}$, and thus $B \geq \norm{W} \geq \norm{\bw}$.
	Let $f$ be the function in $\Fcal$ that corresponds to $\bw$, and let $\bv \in V$ such that $| f(\bx'_{i,j}) - v_{i,j} | \leq \epsilon/L$ for all $i \in [m],~j \in [n]$. Let $\bu \in U$ that corresponds to $\bv$, namely, $u_i = \rho \circ \sigma (v_{i,1},\ldots,v_{i,n})$ for all $i \in [m]$. Note that $| h(\bx_i) - u_i | \leq \epsilon$ for all $i \in [m]$.
Indeed, we have
\[
	| h(\bx_i) - u_i | 
	= \left| \rho \circ \sigma \left(f(\bx'_{i,1}), \ldots, f(\bx'_{i,n})\right) -  \rho \circ \sigma \left( v_{i,1}, \ldots,   v_{i,n}\right) \right|
	\leq  L \cdot \max_{j \in [n]} \left| f(\bx'_{i,j}) - v_{i,j} \right|
	\leq L \cdot \frac{\epsilon}{L}
	= \epsilon~,
\]
where the first inequality follows from the $L$-Lipschitzness of $\rho \circ \sigma$ w.r.t. $\ell_\infty$.
Hence, 
\[
	\Ncal_\infty \left( \Hcal_{B,n,d}^{\sigma, \rho, \Phi}, \epsilon, m \right) 
	\leq | U | 
	\leq | V | \leq \Ncal_\infty(\Fcal, \epsilon/L, mn)~.
\]
Combining the above with \lemref{lem:from_zhang}, and using the fact that the $\Ncal_2$ covering number is at most the $\Ncal_\infty$ covering number (cf. \cite{anthony1999neural}), we get
\begin{align}\label{eq:covering_bound}
	\log \Ncal_2 \left( \Hcal_{B,n,d}^{\sigma, \rho, \Phi}, \epsilon, m \right)  
	&\leq \log \Ncal_\infty \left( \Hcal_{B,n,d}^{\sigma, \rho, \Phi}, \epsilon, m \right)  \nonumber
	\\
	&\leq  \log \Ncal_\infty(\Fcal, \epsilon/L, mn)  \nonumber
	\\
	&\leq \frac{36 b_x^2 B^2}{(\epsilon/L)^2} \log\left( 2mn \lceil 4 b_x B/(\epsilon/L) +2 \rceil +1 \right)~.
\end{align}

Note that for every $\bx \in \Xcal := \left\{\bx \in \reals^d~:~\norm{\phi_j(\bx)} \leq b_x \text{ for all } j \in [n]\right\}$ and $h \in  \Hcal_{B,n,d}^{\sigma, \rho, \Phi}$ we have $| h(\bx) | = | \rho(\sigma(\bw^\top \phi_1(\bx)), \ldots, \sigma(\bw^\top \phi_n(\bx))) | \leq L b_x B$, since $| \bw^\top \phi_j(\bx) | \leq B b_x$, the activation $\sigma$ is $L$-Lipschitz and satisfies $\sigma(0)=0$, and $\rho$ is $1$-Lipschitz w.r.t. $\ell_\infty$ and satisfies $\rho(\mathbf{0})=0$.
By \lemref{lem:from_daniely}, we conclude that 
\[
	\Rcal_m\left( \Hcal_{B,n,d}^{\sigma, \rho, \Phi} \right) 
	\leq L b_x B 2^{-M} + \frac{6 L b_x B}{\sqrt{m}} \sum_{\ell=1}^M 2^{-\ell} \sqrt{\log\Ncal_2\left(\Hcal_{B,n,d}^{\sigma, \rho, \Phi} , L b_x B 2^{-\ell}, m \right)}~,
\]
for every integer $M \geq 1$. By plugging-in $M=\lceil \log(\sqrt{m}) \rceil$ and the expression from \eqref{eq:covering_bound}, we get 
\begin{align*}
	\Rcal_m\left( \Hcal_{B,n,d}^{\sigma, \rho, \Phi} \right) 
	&\leq \frac{L b_x B}{\sqrt{m}} + \frac{6 L b_x B}{\sqrt{m}} \sum_{\ell=1}^{\lceil \log(\sqrt{m}) \rceil} 2^{-\ell} \sqrt{\frac{36 b_x^2 B^2}{(b_x B 2^{-\ell})^2} \log\left( 2mn \lceil 4 b_x B / (b_x B 2^{-\ell}) +2 \rceil +1 \right)}
	\\
	&= \frac{L b_x B}{\sqrt{m}} + \frac{36 L b_x B}{\sqrt{m}} \sum_{\ell=1}^{\lceil \log(\sqrt{m}) \rceil} \sqrt{\log\left( 2mn \lceil 4 \cdot 2^{\ell} +2 \rceil +1 \right)}
	\\
	&\leq \frac{L b_x B}{\sqrt{m}} + \frac{36 L b_x B}{\sqrt{m}} \lceil \log(\sqrt{m}) \rceil \cdot \sqrt{\log\left( 23mn \sqrt{m}\right)}~.
\end{align*}
Hence, for some universal constant $c'>0$ the above is at most
\[
	c' \cdot \frac{L b_x B \log(m) \sqrt{\log\left( mn \right)} }{\sqrt{m}}~.
\]
Requiring this to be at most $\epsilon$ and rearranging, the result follows.

\subsection{Proof of \thmref{thm:lower bound cnn}}

To help the reader track the main proof ideas, we first prove the claim for the case where $B=b_x=1$ and $\epsilon=1/2$ (in Subsection \ref{sec:lower bound cnn simple case}), and then extend the proof for arbitrary $B,b_x,\epsilon > 0$ in Subsection \ref{sec:lower bound cnn general case}.

\subsubsection{Proof for $B=b_x=1$ and $\epsilon=1/2$} \label{sec:lower bound cnn simple case}

Let $m=\log(n)$ and let $d = 3^m$. Consider $m$ points $\bx^1,\ldots,\bx^m$, where for every $i \in [m]$ the point $\bx^i \in \reals^d$ is a vectorization of an order-$m$ tensor $\hat{\bx}^i$ such that each component is indexed by $(j_1,\ldots,j_m) \in [3]^m$. We define the components $x^i_{j_1,\ldots,j_m}$ of $\hat{\bx}^i$ such that  $x^i_{j_1,\ldots,j_m}=1$ if $j_i=3$, and $j_r=2$ for all $r \neq i$, and $x^i_{j_1,\ldots,j_m}=0$ otherwise. 
Note that $\norm{\bx^i}=1$ for all $i \in [m]$. Consider patches of dimensions $2 \times \ldots \times 2$ and stride $1$. Thus, the set $\Phi$ corresponds to all the patches of dimensions $2 \times \ldots \times 2$ in the tensor. Note that there are $2^m=n$ such patches. Indeed, given an index $(j_1,\ldots,j_m) \in [2]^m$, we can define a patch which contains the indices $\left\{(j_1,\ldots,j_m) + (\Delta_1,\ldots,\Delta_m):~(\Delta_1,\ldots,\Delta_m) \in \{0,1\}^m \right\}$. We say that $(j_1,\ldots,j_m)$ is the \emph{base index} of this patch. Note that each $(j_1,\ldots,j_m) \in [2]^m$ is a base index of exactly one patch. Also, an index $(j_1,\ldots,j_m)$ which includes some $r \in [m]$ with $j_r=3$ does not induce a patch of the form $\left\{(j_1,\ldots,j_m) + (\Delta_1,\ldots,\Delta_m):~(\Delta_1,\ldots,\Delta_m) \in \{0,1\}^m \right\}$, since for $\Delta_r=1$ we get an invalid index.

We show that for any $\by \in \{0,1\}^m$ we can find a filter $\bw$, such that $\bw$ is an order-$m$ tensor of dimensions $2 \times \ldots \times 2$ and satisfies the following. Let $N_\bw$ be the neural network that consists of a convolutional layer with the patches $\Phi$ and the filter $\bw$, followed by a max-pooling layer. Then, $N_\bw(\bx^i)=y_i$ for all $i \in [m]$.  Thus, we can shatter $\bx^1,\ldots,\bx^m$ with margin $\epsilon=1/2$. Moreover, the spectral norm of the matrix $W$ that corresponds to the convolutional layer is at most $1$.


Consider the filter $\bw$ of dimensions $2 \times \ldots \times 2$ such that $w_{j_1,\ldots,j_m}=1$ if $(j_1,\ldots,j_m) = \mathbf{1} + \by$, and $w_{j_1,\ldots,j_m}=0$ otherwise.
We now show that 
$N_\bw(\bx^i)=y_i$ for all $i \in [m]$. 
Since the filter $\bw$ has a single non-zero component, then the inner product between $\bw$ and a patch of $\bx^i$ is non-zero iff the patch of $\bx^i$ has a non-zero component in the appropriate position. More precisely, for a patch with base index $(j_1,\ldots,j_m)$, the inner product between the components of $\bx^i$ in the indices of the patch and the filter $\bw$ is $1$ iff $x^i_{(j_1,\ldots,j_m) + \by} = 1$, and otherwise the inner product is $0$. Since $x^i_{q_1,\ldots,q_m} = 1$ iff $q_i=3$ and $q_r=2$ for $r \neq i$, then  $x^i_{(j_1,\ldots,j_m) + \by} = 1$ iff $j_i = 3-y_i$ and $j_r = 2-y_r$ for $r \neq i$. Now, if $y_i=0$ then there is no patch such that the base index satisfies $j_i = 3-y_i = 3$, since all base indices are in $[2]^m$, and therefore $N_\bw(\bx^i)=0$. If $y_i=1$ then the patch whose base index satisfies $j_i = 3-y_i$ and $j_r = 2-y_r$ for $r \neq i$ gives output $1$ (and all other patches give output $0$) and hence $N_\bw(\bx^i)=1$. Thus, we have $N_\bw(\bx^i)=y_i$ as required.

For example, consider the case where $m=2$. Then, the tensor $\hat{\bx}^1$ is the matrix
\[
\hat{\bx}^1 = \begin{bmatrix} 
0 & 0 & 0 \\
0 & 0 & 0 \\
0 & 1 & 0
\end{bmatrix}. 
\]
For $\by = (1,1)^\top$ we have $\bw = \begin{bmatrix} 0 & 0 \\ 0 &1 \end{bmatrix}$ and hence the patch with base index $(2,1)$ gives output $1$.  For $\by = (1,0)^\top$ we have $\bw = \begin{bmatrix} 0 & 0 \\ 1 & 0 \end{bmatrix}$ and hence the patch with base index $(2,2)$ gives output $1$. However, for  $\by = (0,1)^\top$ we have $\bw = \begin{bmatrix} 0 & 1 \\ 0 & 0 \end{bmatrix}$ and hence there is no patch that gives output $1$. Thus, in all the above cases we have $N_\bw(\bx^1)=y_1$.

It remains to show that the spectral norm of the matrix $W$ that corresponds to the convolutional layer with the filter $\bw$ is at most $1$. Thus, we show that for every input $\bx \in \reals^d$ with $\norm{\bx}=1$ the inputs to the hidden layer is a vector with norm at most $1$. We view $\bx$ as the vectorization of a tensor $\hat{\bx}$ with components $x_{j_1,\ldots,j_m}$ for $(j_1,\ldots,j_m) \in [3]^m$. Since the filter $\bw$ contains a single $1$-component and all other components are $0$, then the input to each hidden neuron is a different component of $\hat{\bx}$. More precisely, since the filter $\bw$ contains $1$ at index $\mathbf{1}+\by$ then for the patch with base index $(j_1,\ldots,j_m)$ the corresponding hidden neuron has input $x_{(j_1,\ldots,j_m)+\by}$. Note that each hidden neuron corresponds to a different base index and hence the input to each hidden neuron is a different component of $\hat{\bx}$. Therefore, the norm of the vector whose components are the inputs to the hidden neurons is at most the norm of the input $\bx$, and hence it is at most $1$.

\subsubsection{Proof for arbitrary $B,b_x,\epsilon>0$}\label{sec:lower bound cnn general case}

Let $m =  \left( \frac{b_x B}{2\epsilon} \right)^2 \cdot \log(n)$ and let $d= \left( \frac{b_x B}{2\epsilon} \right)^2 \cdot 3^{\log(n)}$. Let $m'=\log(n)$ and let $L= \left( \frac{b_x B}{2\epsilon} \right)^2$. Consider $m$ points $\bx^1,\ldots,\bx^m$, where for every $i \in [m]$ the point $\bx^i \in \reals^d$ is a vectorization of a tensor $\hat{\bx}^i$ of order $m'+1$, such that each component is indexed by $(j_1,\ldots,j_{m'},\ell) \in [3]^{m'} \times [L]$.  Consider a partition of $[m]$ into $L$ disjoint susets $S_1,\ldots,S_L$, each of size $m/L = m'$.

We define the components $x^i_{j_1,\ldots,j_{m'},\ell}$ of $\hat{\bx}^i$ as follows: Suppose that $i \in S_r := \{k_1,\ldots,k_{m'} \}$ for some $r \in L$, and that $i=k_t$, i.e., $i$ is the $t$-th element in the subset $S_r$. For every $\ell \neq r$ we define $x^i_{j_1,\ldots,j_{m'},\ell}=0$ for every $j_1,\ldots,j_{m'} \in [3]^{m'}$, namely, if $\ell$ does not correspond to the subset of $i$ then the component is $0$. For $\ell = r$ the component $x^i_{j_1,\ldots,j_{m'},\ell}$ is defined in a similar way to the tensor $\hat{\bx}^i$ from Subsection~\ref{sec:lower bound cnn simple case}, but with respect to the subset $S_r$ and at scale $b_x$. Formally, for $\ell = r$ we have $x^i_{j_1,\ldots,j_{m'},\ell}=b_x$ if $j_t=3$, and $j_k=2$ for all $k \neq t$, and $x^i_{j_1,\ldots,j_{m'},\ell}=0$ otherwise.
Note that $\norm{\bx^i}=b_x$ for all $i \in [m]$.

\begin{sloppypar}
Consider patches of dimensions $2 \times \ldots \times 2 \times L$ and stride $1$. Thus, the set $\Phi$ corresponds to all the patches of dimensions $2 \times \ldots \times 2 \times L$ in the tensor. Note that since the last dimension is $L$, then the filter can ``move" only in the first $m'$ dimensions. Also, note that there are $2^{m'}=n$ such patches. Indeed, given $(j_1,\ldots,j_{m'}) \in [2]^{m'}$, we can define a patch which contains the indices $\left\{(j_1,\ldots,j_{m'},0) + (\Delta_1,\ldots,\Delta_{m'},\Delta_{m'+1}):~(\Delta_1,\ldots,\Delta_{m'}) \in \{0,1\}^{m'},~\Delta_{m'+1} \in [L] \right\}$. We say that $(j_1,\ldots,j_{m'})$ is the \emph{base index} of this patch. Note that each $(j_1,\ldots,j_{m'}) \in [2]^{m'}$ is a base index of exactly one patch. Also, if $(j_1,\ldots,j_{m'})$ includes some $r \in [m']$ with $j_r=3$ then it does not induce a patch of the form $\left\{(j_1,\ldots,j_{m'},0) + (\Delta_1,\ldots,\Delta_{m'},\Delta_{m'+1}):~(\Delta_1,\ldots,\Delta_{m'}) \in \{0,1\}^{m'},~\Delta_{m'+1} \in [L] \right\}$, since for $\Delta_r=1$ we get an invalid index.
\end{sloppypar}

We show that for any $\by \in \{0,1\}^m$ we can find a filter $\bw$, such that $\bw$ is an order-$(m'+1)$ tensor of dimensions $2 \times \ldots \times 2 \times L$ and satisfies the following. Let $N_\bw$ be the neural network that consists of a convolutional layer with the patches $\Phi$ and the filter $\bw$, followed by a max-pooling layer. Then, for all $i \in [m]$ we have: if $y_i=0$ then $N_\bw(\bx^i)=0$, and if $y_i=1$ then $N_\bw(\bx^i)=2\epsilon$.  Thus, we can shatter $\bx^1,\ldots,\bx^m$ with margin $\epsilon$. Moreover, the spectral norm of the matrix $W$ that corresponds to the convolutional layer is at most $B$.

We now define the filter $\bw$ of dimensions $2 \times \ldots \times 2 \times L$. For every $\ell \in [L]$ we define the components $w_{j_1,\ldots,j_{m'},\ell}$ as follows. Let $\by_{S_\ell} \in \{0,1\}^{m'}$ be the restriction of $\by$ to the indices in $S_\ell$. Then, $w_{j_1,\ldots,j_{m'},\ell}=\frac{2\epsilon}{b_x}$ if 
$(j_1,\ldots,j_{m'}) = \mathbf{1} + \by_{S_\ell}$, and $w_{j_1,\ldots,j_{m'},\ell}=0$ otherwise.
We show that for all $i \in [m]$, if $y_i=0$ then $N_\bw(\bx^i)=0$, and if $y_i=1$ then $N_\bw(\bx^i)=2\epsilon$.
Suppose that $i \in S_r := \{k_1,\ldots,k_{m'} \}$ for some $r \in L$, and that $i=k_t$, i.e., $i$ is the $t$-th element in the subset $S_r$. Then, the tensor $\hat{\bx}^i$ has a non-zero component only at $x^i_{j_1,\ldots,j_{m'},r}$ with $j_t=3$, and $j_s=2$ for all $s \neq t$. Moreover, the filter $\bw$ has a non-zero component at index $(q_1,\ldots,q_{m'},r)$ iff $(q_1,\ldots,q_{m'}) = \mathbf{1} + \by_{S_r}$.
Hence, the inner product between $\bw$ and a patch of $\bx^i$ is non-zero iff the patch has a base index $(j_1,\ldots,j_{m'})$ such that $(j_1,\ldots,j_{m'}) + \by_{S_r} = (p_1,\ldots,p_{m'})$ where $p_t=3$, and $p_s=2$ for all $s \neq t$. If $y_i=0$ then the $t$-th component of $\by_{S_r}$ is $0$, and there is no patch such that the base index satisfies $j_t + (\by_{S_r})_t = j_t + 0 = p_t = 3$. Therefore, $N_\bw(\bx^i)=0$. If $y_i=1$ then the patch whose base index satisfies $j_t = 3 - (\by_{S_r})_t = 3 - 1 = 2$, and $j_s = 2 - (\by_{S_r})_s$ for $s \neq t$, gives output $\frac{2\epsilon}{b_x} \cdot b_x = 2\epsilon$ (and all other patches give output $0$).

It remains to show that the spectral norm of the matrix $W$ that corresponds to the convolutional layer with the filter $\bw$ is at most $B$. Thus, we show that for every input $\bx \in \reals^d$ with $\norm{\bx}=1$ the inputs to the hidden layer are a vector with norm at most $B$. We view $\bx$ as the vectorization of a tensor $\hat{\bx}$ with components $x_{j_1,\ldots,j_{m'},\ell}$ for $(j_1,\ldots,j_{m'},\ell) \in [3]^{m'} \times [L]$. The inner product between a patch of $\bx$ and the filter $\bw$ can be written as 
\[
	\sum_{\ell \in [L]}\frac{2\epsilon}{b_x} \cdot x_{q^{(\ell)}_1,\ldots,q^{(\ell)}_{m'},\ell}~.
\] 
Thus, for each $\ell$ there is a single index of $\hat{\bx}$ that contributes to the inner product, since for every $\ell$ the filter $\bw$ has a single non-zero component, which equals $\frac{2\epsilon}{b_x}$. By Cauchy–Schwarz, the above sum is at most 
\begin{equation} \label{eq:bound spectral}
	\frac{2\epsilon}{b_x} \cdot \sqrt{L} \cdot \sqrt{\sum_{\ell \in [L]} x^2_{q^{(\ell)}_1,\ldots,q^{(\ell)}_{m'},\ell}}
	= \frac{2\epsilon}{b_x} \cdot \frac{b_x B}{2\epsilon} \cdot \sqrt{\sum_{\ell \in [L]} x^2_{q^{(\ell)}_1,\ldots,q^{(\ell)}_{m'},\ell}}
	= B \cdot \sqrt{\sum_{\ell \in [L]} x^2_{q^{(\ell)}_1,\ldots,q^{(\ell)}_{m'},\ell}}~.
\end{equation}
Hence, the input to the hidden neuron that corresponds to the patch is bounded by the above expression. 
Moreover, since for every $\ell \in [L]$ the filter $\bw$ has a single non-zero component such that the last dimension of its index is $\ell$, then for every two patches with different base indices, the bound in the above expression includes different indices of $\hat{\bx}$. Namely, if the inner product between one patch of $\bx$ and the filter $\bw$ is $\sum_{\ell \in [L]}\frac{2\epsilon}{b_x} \cdot x_{q^{(\ell)}_1,\ldots,q^{(\ell)}_{m'},\ell}$ and the inner product between another patch of $\bx$ and the filter $\bw$ is $\sum_{\ell \in [L]}\frac{2\epsilon}{b_x} \cdot x_{p^{(\ell)}_1,\ldots,p^{(\ell)}_{m'},\ell}$, then for every $\ell$ we have $(q^{(\ell)}_1,\ldots,q^{(\ell)}_{m'}) \neq (p^{(\ell)}_1,\ldots,p^{(\ell)}_{m'})$. Since by \eqref{eq:bound spectral} the square of the input to each hidden neuron can be bounded by $B^2 \cdot \sum_{\ell \in [L]} x^2_{q^{(\ell)}_1,\ldots,q^{(\ell)}_{m'},\ell}$ for some subset $\left\{x_{q^{(\ell)}_1,\ldots,q^{(\ell)}_{m'},\ell}\right\}_{\ell \in [L]}$ of components, and since for each two hidden neurons these subsets are disjoint, then the norm of the vector of inputs to the hidden neurons can be bounded by 
\[
	\sqrt{B^2 \cdot \sum_{k \in [d]} x_k^2} 
	\leq	\sqrt{B^2 \cdot 1}
	= B~.
\]

\end{document}